%% file: main.tex
\pdfoutput=1
\documentclass[11pt]{article}
\usepackage[T1]{fontenc}
\usepackage[utf8]{inputenc}
\usepackage{lmodern}
\usepackage{times}
\usepackage{microtype}
\usepackage{preprint}
\usepackage{amsmath,amssymb,amsthm,mathtools}
\usepackage{booktabs,array}
\usepackage[inline]{enumitem}
\usepackage{graphicx}
\usepackage{xcolor}
\usepackage{tikz,pgfplots}
\pgfplotsset{compat=1.17}
\usetikzlibrary{arrows.meta,positioning,calc}
\definecolor{skmu}{HTML}{17365D}
\definecolor{sknu}{HTML}{167D8D}
\definecolor{skacc}{HTML}{C56923}
\usepackage[round]{natbib}

\PassOptionsToPackage{hyphens}{url}\usepackage{hyperref}
\usepackage[nameinlink,noabbrev,capitalize]{cleveref}
\definecolor{linkblue}{RGB}{0,64,160}
\hypersetup{colorlinks=true,linkcolor=linkblue,citecolor=linkblue,urlcolor=linkblue}

\newtheorem{theorem}{Theorem}
\newtheorem{proposition}[theorem]{Proposition}
\newtheorem{corollary}[theorem]{Corollary}
\newtheorem{lemma}[theorem]{Lemma}
\theoremstyle{definition}
\newtheorem{definition}[theorem]{Definition}
\newtheorem{example}[theorem]{Example}
\theoremstyle{remark}

\newcommand{\R}{\mathbb{R}}
\newcommand{\E}{\mathbb{E}}
\newcommand{\Sph}{\mathbb{S}^{d-1}}
\newcommand{\Pcal}{\mathcal{P}}
\newcommand{\Scal}{\mathcal{S}}
\newcommand{\Ccal}{\mathcal{C}}
\newcommand{\Var}{\operatorname{Var}}
\newcommand{\supp}{\operatorname{supp}}
\newcommand{\SW}{\mathrm{SW}}
\newcommand{\norm}[1]{\lVert #1\rVert}
\newcommand{\abs}[1]{\lvert #1\rvert}
\DeclareMathOperator*{\argmin}{arg\,min}

\begin{document}

\title{\textbf{Sliced Wasserstein Sketches with Error Certificates}}
\author[1]{Elaheh Akbari}
\author[1]{James Liu}
\author[2]{ Keaton Hamm}
\author[1]{Soheil Kolouri}

\affil[1]{College of Connected Computing, Vanderbilt University}
\affil[2]{Department of Mathematics, Division of Data Science, University of Texas at Arlington}
\affil[ ]{\vspace{0.4em} \texttt{elaheh.akbari@vanderbilt.edu}}

\date{}
\maketitle

\begin{abstract}
Compact representations that support geometrically meaningful comparisons are valuable for learning from probability distributions. We study optimal, free-weight $M$-atom quantization of one-dimensional probability measures in quadratic Wasserstein distance and extend it to higher dimensions via slicing. Using $L$ shared projections yields a sketch with $O(LM)$ comparison time. Our main result is an exact decomposition of squared distance error into residual and partition-mismatch terms, yielding computable certificates for arbitrary centroid sketches. For optimal sketches, a finite $(2+\delta)$ moment for some $\delta>0$ guarantees $O(M^{-1})$ error in the finite-direction sliced distance, while equal-mass sketches can converge more slowly. Each slice preserves its mean, while storing residual energies enables covariance recovery from sufficiently many generic directions. Experiments on document, point-cloud, and flow-cytometry benchmarks show that optimal sketches substantially reduce storage while preserving the classification performance of uncompressed sliced distances. They also improve generalization from few labelled subjects and support end-to-end learning as compact pooling layers.
\end{abstract}

\section{Introduction}\label{sec:intro}

\begin{figure}[t!]
\centering
\begin{minipage}[c]{0.5\textwidth}
\centering
\includegraphics[width=\linewidth]{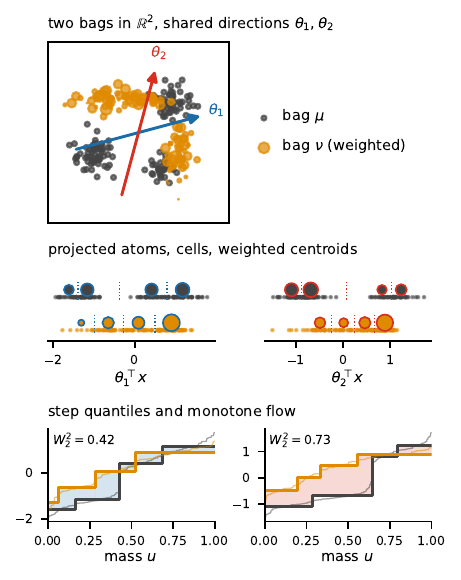}
\end{minipage}\hfill
\begin{minipage}[c]{0.46\textwidth}
\caption{Sketch construction. Top: two bags in $\R^2$ (one weighted) and two shared directions. Middle: projected atoms, $M=4$ optimal contiguous cells (dotted), and weighted centroids per direction; marker area is proportional to cell mass, and cells differ between bags. Bottom: sketch step quantiles over mass $u$, with original quantiles in the background and shading marking coupled pairs at each mass level. This monotone coupling has quadratic cost $W_2^2(\hat\mu_\ell,\hat\nu_\ell)$; the squared sketch distance $D_{\Theta,M}^2$ averages these costs over directions.}
\label{fig:construction}
\end{minipage}
\vspace{-0.8\baselineskip}
\end{figure}

A document represented by word embeddings and a shape represented by a point cloud are naturally compared as probability distributions. In retrieval, clustering, and other tasks involving many such objects, repeatedly computing transport distances on the original samples can dominate the cost. A useful summary should therefore be built once, occupy little memory, support many cheap comparisons, and quantify how far those comparisons can deviate from the true distances.

Sliced Wasserstein distances, introduced by \citet{rabin2012}, compare distributions by averaging one-dimensional transport costs over projection directions. In one dimension, quadratic Wasserstein distance is the $L^2$ distance between quantile functions. This representation underlies sketches based on quantile grids \citep{kolouri2015radon,naderializadeh2021,lu2024,carpinteroperez2024}, streaming rank summaries \citep{nguyen2025streamsw}, and parametric fits \citep{nadjahi2021,liu2025emperor}. A fixed grid of quantile levels is particularly convenient because comparison reduces to a Euclidean distance. However, its allocation of resolution is less flexible because cells of equal probability can cover very different spatial scales, and tail cells or cells crossing gaps in the support can dominate the approximation error.

In this work, we instead summarize each projection by the measure supported on at most $M$ atoms that minimizes its quadratic Wasserstein error. Both atom locations and weights are free to adapt to the projected distribution. For a weighted point measure, this is exactly one-dimensional weighted $k$-means, which admits a globally optimal dynamic program \citep{wu1991,gronlund2017}. Applying it on $L$ shared directions gives a reusable sketch with $O(LM)$ storage and comparison time. The construction also returns the quantization error of every slice, and it is exact on any projection with at most $M$ distinct values. Figure~\ref{fig:construction} illustrates the construction and comparison.

The central question is how these individually optimal summaries affect the distance \emph{between} two distributions. With a shared partition of the quantile domain, replacing each quantile by its cell means is an orthogonal projection and contracts $W_2$. Adaptive quantizers generally choose different partitions, so this contraction argument no longer applies. The triangle inequality still bounds distance error by the sum of the two quantization errors, but does not describe the effect of the mismatched cells. We resolve this gap with an exact identity for arbitrary \emph{centroid sketches}: summaries whose atoms are the means of their quantile cells. The identity isolates the boundary terms that can expand distances and yields sketch-computable certificates.

Our contributions are:
\begin{itemize}[leftmargin=*,itemsep=1pt,topsep=2pt]
\item \textbf{Distance certificates for adaptive partitions.} The partition-mismatch identity (\cref{thm:identity}) yields two-sided bounds, shared-partition contraction, and $O(M^{-2})$ squared-distance error when quantiles are Lipschitz and cell masses are $O(M^{-1})$ (\cref{cor:certificates,cor:lipschitz}). An explicit family shows that a larger $M^{-3/2}$ error is possible for general centroid sketches (\cref{ex:sharp}).
\item \textbf{Sliced approximation rates.} Combining the certificates with classical quantization bounds gives $O(M^{-1})$ error in the finite-direction sliced distance under a finite $(2+\delta)$ moment for some $\delta>0$ (\cref{thm:sliced,thm:rate}). Gaussian projections and projections of uniform balls in dimensions $d\ge3$ exhibit slower equal-mass rates, illustrating the benefit of adaptive weights (\cref{prop:equalmass}).
\item \textbf{Moment recovery.} Classical centroid identities give exact directional means and a second-moment deficit equal to the residual energy. We use these identities to recover means and covariances from finitely many directions, and to bound higher-moment tensor errors, with a per-cell correction that improves the rate (\cref{prop:moments}).
\item \textbf{Compression at matched accuracy.} On document, point-cloud, and flow-cytometry benchmarks the sketch reproduces the accuracy of the uncompressed sliced distance with $1.5$ to more than $3000$ times fewer stored numbers, and learned slices make the sketched distance competitive with full-dimensional transport (\cref{sec:experiments}). We also report where equal-mass cells are the better centroid sketch (\cref{app:sweeps}).
\end{itemize}
The exact solver and underlying quantization and moment identities are established results. Our analysis concerns their use in distribution sketches and, particularly, the effect of mismatched partitions on pairwise distances. The target is the sliced distance on the chosen directions; approximation of population sliced Wasserstein and comparison with $W_2$ are separate questions \citep{nadjahi2020,bonnotte2013}.

\section{Related Work}\label{sec:related}

\noindent\textbf{Exact one-dimensional quantization.}
Free-weight $M$-atom $W_2$ quantization of a weighted point measure is equivalent to one-dimensional weighted $k$-means with contiguous clusters. An early dynamic program for scalar quantization is due to \citet{bruce1964}; \citet{wu1991} obtained $O(MN)$ time after sorting using matrix searching \citep{aggarwal1987}. \citet{gronlund2017} show how to achieve $O(MN)$ time after sorting using only $O(N)$ space, and also describe a divide-and-conquer alternative with $O(MN\log N)$ running time. Lloyd--Max iterations provide an alternative to dynamic programming by repeatedly updating cell boundaries and centroids \citep{max1960,lloyd1982}. Whether they recover a global optimum depends on the distribution. \citet{fleischer1964} studied sufficient conditions for optimality, and subsequent work showed that log-concave densities admit a unique optimal quantizer to which the iterations converge \citep{trushkin1982,kieffer1983}. For more general distributions, however, cell boundaries and centroids can stabilize at a suboptimal partition.

\noindent\textbf{Quantization theory.}
Classical centroid stationarity \citep{max1960,lloyd1982} yields the numerical-integration bound $\tfrac12\norm{f''}_\infty\varepsilon_\mu^2$ for residual energy $\varepsilon_\mu^2$ \citep{pages1998,pages2003}, with extensions by \citet{lemaire2020}. We use Zador's asymptotic formula \citep{zador1963,zador1982} in its general formulation \citep{grafLuschgy2000} and the extended Pierce bound \citep{pierce1970,luschgy2008}. Separate theories address equal-mass rates \citep{giles2019,xu2019,bencheikh2022} and empirical-quantizer consistency and stability \citep{bartlett1998,levrard2015,liu2020}. Gauss quadrature matches polynomial moments \citep[Section~3.5(v)]{dlmf}, with nodes and weights computable by an eigenvalue algorithm \citep{golub1969}. %

\noindent\textbf{Sliced transport and compressed comparisons.}
Sliced transport \citep{rabin2012} led to sliced and Radon barycenters \citep{bonneel2015}, analytic comparisons with full-dimensional Wasserstein distances \citep{bonnotte2013}, and statistical and projection-sampling guarantees \citep{nadjahi2020,nietert2022,nguyen2024,tran2026}. Closest to our analysis are Gaussian projection summaries \citep{nadjahi2021}, transport between $k$-means summaries \citep{beugnot2021}, common-codebook quantization \citep{gachon2025}, and Voronoi and other structured approximations of individual measures \citep{hamm2026structured}. These approaches control individual approximation errors, yielding pairwise-distance bounds through the triangle inequality. Our identity instead resolves residual and boundary contributions for centroid sketches with different mass partitions.

\noindent\textbf{Distribution sketches and embeddings.}
Stream-SW \citep{nguyen2025streamsw} combines shared directions, KLL rank sketches \citep{karnin2016}, and exact one-dimensional transport, supporting streaming updates with probabilistic bounds involving the support diameter. We instead solve offline quantization and record its distortion. Sliced Wasserstein kernels and continuous feature maps \citep{kolouri2016} precede finite quantile-grid embeddings with Euclidean comparison \citep{naderializadeh2021,lu2024,carpinteroperez2024}; the cumulative distribution transform uses transport maps relative to a reference \citep{park2018}. Equal-mass cell means minimize quadratic transport error under prescribed equal weights \citep{xu2019,bencheikh2022}. Related summaries include midpoint grids and smooth piecewise-linear reconstruction for federated statistics \citep{charpentier2026}, geometric coresets using one-dimensional clustering \citep{harpeled2007}, and full-dimensional uniform-weight Wasserstein coresets \citep{claici2020}. Directional moment reconstruction is polynomial recovery \citep{bodmann2018}, combined with GMM summaries and least squares by \citet{liu2025emperor}; it differs from identifying an entire law through projections \citep{cramerwold1936,cuestaalbertos2007}.

\section{$W_2$-Optimal Sketches}\label{sec:method}

\noindent\textbf{Notation.}
Let $\Pcal_2(\R^d)$ denote probability measures with finite second moment. For $\mu\in\Pcal_2(\R)$ let $F_\mu$ be its distribution function and $Q_\mu(t)=\inf\{x:F_\mu(x)\ge t\}$, $t\in(0,1)$, its quantile function, so that $Q_\mu\in L^2(0,1)$ is nondecreasing and, for $\mu,\nu\in\Pcal_2(\R)$,
\begin{equation}\label{eq:W2quantile}
 W_2^2(\mu,\nu)=\int_0^1\bigl(Q_\mu(t)-Q_\nu(t)\bigr)^2\,dt=\norm{Q_\mu-Q_\nu}^2 ,
\end{equation}
where $\norm{\cdot}$ and $\langle\cdot,\cdot\rangle$ denote the $L^2(0,1)$ norm and inner product. This is the classical monotone-coupling formula \citep{dallaglio1956}; see \citet[Chapter~2]{santambrogio2015} for a modern treatment.

\subsection{One-dimensional building block}
For $\mu\in\Pcal_2(\R)$ and $M\ge1$, we denote an \emph{optimal $M$-atom quantizer} by $\Ccal_M(\mu)$, with optimal error $\varepsilon_M^\star(\mu)$:
\begin{equation}\label{eq:QM}
\begin{aligned}
 \Ccal_M(\mu)&\in\argmin\{W_2^2(\mu,\eta):\ \abs{\supp\eta}\le M\},\\
 \varepsilon_M^\star(\mu)&=W_2\bigl(\mu,\Ccal_M(\mu)\bigr).
\end{aligned}
\end{equation}
Write $\Ccal_M(\mu)=\sum_{j=1}^K q_j\delta_{b_j}$, where $K\le M$, $q_j>0$, and $b_1<\dots<b_K$. Nearest-neighbor cells are intervals, separated by midpoints of adjacent atoms; a consistent assignment handles ties. Their preimages under $Q_\mu$ are consecutive mass intervals $I_j=(s_{j-1},s_j]$. Optimality gives the centroid condition
\[
 b_j=\frac1{q_j}\int_{I_j}Q_\mu(x)\,dx,\qquad q_j=s_j-s_{j-1}.
\]
For weighted input $\mu=\sum_{i=1}^N w_i\delta_{x_i}$, the problem is exactly weighted $k$-means on $\R$ after sorting and merging duplicate locations. Exact dynamic programming takes $O(MN)$ time after sorting using SMAWK to exploit the Monge cost structure \citep{wu1991,gronlund2017,song2020}. Our implementation uses the simpler $O(MN\log N)$ divide-and-conquer variant, retaining $O(MN)$ traceback storage per slice, and includes an $O(MN^2)$ oracle for verification (\cref{app:impl}). This gives a global optimum of the discrete objective, whereas Lloyd--Max iterations need additional assumptions for such a guarantee \citep{fleischer1964,trushkin1982,kieffer1983}. \Cref{fig:freevsequal} compares prescribed and optimized cell masses on a real projection.

\begin{figure}[t]
\centering
\begin{minipage}[c]{0.5\textwidth}
\centering
\resizebox{\linewidth}{!}{\input{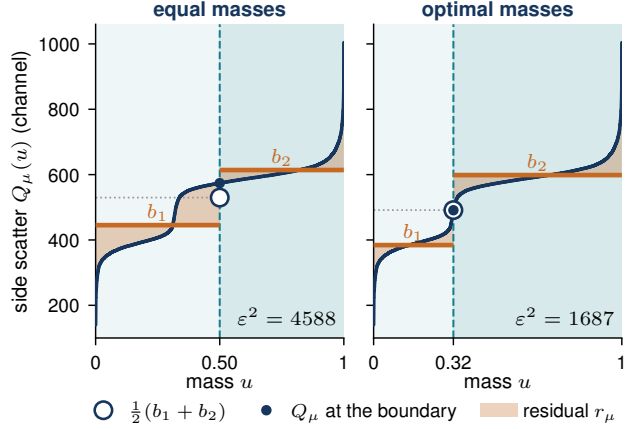}}
\end{minipage}\hfill
\begin{minipage}[c]{0.46\textwidth}
\caption{Equal-mass and optimal cells ($M=2$) for one AML cytometry subject’s side-scatter channel ($10^4$ cells; \cref{sec:experiments}). Atoms are cell means of $Q_\mu$; the shaded residual has energy $\varepsilon^2$. At the equal-mass boundary, $Q_\mu$ differs from the atom midpoint, so the cells are not nearest-neighbour cells. The optimal boundary $s_*=0.32$ satisfies the midpoint condition and lowers energy by $63\%$.}
\label{fig:freevsequal}
\end{minipage}
\end{figure}

\subsection{Centroid sketches}
The distance analysis uses only the centroid condition, covering summaries with prescribed mass partitions.
\begin{definition}[centroid sketch]\label{def:sketch}
A \emph{centroid sketch} of $\mu\in\Pcal_2(\R)$ with $M$ cells is a partition $0=s_0<s_1<\dots<s_M=1$ of $(0,1)$ into intervals $I_j=(s_{j-1},s_j]$ together with the cell means $b_j=q_j^{-1}\int_{I_j}Q_\mu$, $q_j=s_j-s_{j-1}$. Its step quantile is $\hat Q_\mu=\sum_j b_j\mathbf 1_{I_j}$, its measure is $\hat\mu=\sum_j q_j\delta_{b_j}$, its \emph{residual} is $r_\mu=Q_\mu-\hat Q_\mu$, its \emph{quantization error} is $\varepsilon_\mu=\norm{r_\mu}=W_2(\mu,\hat\mu)$, and its \emph{residual energy} is $\varepsilon_\mu^2$.
\end{definition}
Since $Q_\mu$ is nondecreasing, so is $\hat Q_\mu$, and moreover $\hat Q_\mu = Q_{\hat\mu}$. The monotone coupling gives $W_2(\mu,\hat\mu)=\norm{r_\mu}$, and the residual integrates to zero on every cell. We use $\varepsilon_\mu$ to denote the chosen sketch, suppressing its dependence on the partition and cell count. Every $M$-cell sketch satisfies $\varepsilon_\mu\ge\varepsilon_M^\star(\mu)$. Equality holds for the optimal quantizer $\Ccal_M(\mu)$, whose cells are optimal clusters. The \emph{equal-mass} sketch $s_j=j/M$ (the optimal equal-weight approximation; \citealp{xu2019,bencheikh2022}) and all Lloyd fixed points are centroid sketches.

\subsection{Sliced sketch and comparison}
For $\mu\in\Pcal_2(\R^d)$ and a fixed set of \emph{shared} directions $\Theta=(\theta_1,\dots,\theta_L)\subset\Sph$, write $\mu_\ell=(\theta_\ell^\top\cdot)_\#\mu$ and define
\begin{equation}\label{eq:sketch}
\begin{aligned}
 \Scal_{\Theta,M}(\mu)&=\bigl(\Ccal_M(\mu_1),\dots,\Ccal_M(\mu_L)\bigr),\\
 D_{\Theta,M}^2(\mu,\nu)&=\frac1L\sum_{\ell=1}^L W_2^2\bigl(\Ccal_M(\mu_\ell),\Ccal_M(\nu_\ell)\bigr).
\end{aligned}
\end{equation}
With $M$ atoms per slice, the sketch stores $L(2M-1)$ scalars: atom locations and all but one weight per slice, in addition to the shared directions. Recording one residual energy per slice increases this to $2LM$ scalars; recording all per-cell energies instead requires $L(3M-1)$. Each projection is quantized independently, and comparisons use the resulting one-dimensional summaries directly, making it ideal for parallel computing.

Each summand in \eqref{eq:sketch} compares sorted weighted atoms. Its monotone coupling has at most $2M-1$ nonzero flows and is computed by a two-pointer sweep in $O(M)$ time, giving $O(LM)$ per query. For $N$ input points, construction using the SMAWK-based exact solver takes $O\bigl(LNd+LN\log N+LMN\bigr)$ time, with the three terms accounting for projection, sorting, and quantization, respectively \citep{wu1991,gronlund2017}. We note that the divide-and-conquer implementation in \citet{gronlund2017}  takes $O(LMN\log N)$ for the quantization step.

The reference distance uses the same directions:
\[
\SW_{2,\Theta}^2(\mu,\nu)=\frac1L\sum_{\ell=1}^L W_2^2(\mu_\ell,\nu_\ell).
\]
Section~\ref{sec:theory} bounds its discrepancy from $D_{\Theta,M}$. This separates sketching error from the approximation incurred by using finitely many directions \citep{nadjahi2020}.

\noindent\textbf{Choice of slices.}
The distance analysis of \cref{sec:theory} is carried out slice by slice and uses no property of the maps $x\mapsto\theta_\ell^\top x$ beyond their being shared by all measures. The identity, the certificates, and \cref{thm:sliced} therefore hold unchanged when the directions are learned from data, or when the linear maps are replaced by shared measurable maps $g_\ell:\R^d\to\R$ with $g_{\ell\#}\mu\in\Pcal_2(\R)$; the reference is then the sliced distance on those slices. \Cref{thm:rate} holds for learned unit directions. The moment recovery of \cref{prop:moments} requires linear slices in general position.

\section{Theory}\label{sec:theory}
Throughout, $\mu,\nu\in\Pcal_2(\R)$ have centroid sketches $\hat\mu, \hat\nu$ with cells $I_j=(s_{j-1},s_j]$, atoms $b_j$ (for $\hat\mu$) and $J_k=(t_{k-1},t_k]$, atoms $c_k$ (for $\hat\nu$), written with $M$ cells each to simplify notation. For unequal cell counts, each sum runs over the boundaries of its own sketch. Let $D^2=W_2^2(\hat\mu,\hat\nu)=\norm{\hat Q_\mu-\hat Q_\nu}^2$ be the squared sketch distance and
\begin{equation}\label{eq:phi}
 \varphi_\mu(s)=\int_0^s r_\mu(x)\,dx,\qquad \varphi_\nu(t)=\int_0^t r_\nu(x)\,dx .
\end{equation}
Because each residual has zero mean on each of its own cells, $\varphi_\mu$ vanishes at every $s_j$ and $\varphi_\nu$ at every $t_k$; for $s\in J_k$ one has $\varphi_\nu(s)=\int_{t_{k-1}}^s r_\nu=-\int_s^{t_k} r_\nu$.

\subsection{Exact sketch distance identity}
\begin{theorem}[partition-mismatch identity]\label{thm:identity}
For any two centroid sketches of $\mu,\nu\in\Pcal_2(\R)$,
\begin{align}\label{eq:identity}
 W_2^2(\mu,\nu)-D^2
 =&\norm{r_\mu-r_\nu}^2+2\sum_{j=1}^{M-1}(b_{j+1}-b_j)\,\varphi_\nu(s_j) \nonumber\\
 &+2\sum_{k=1}^{M-1}(c_{k+1}-c_k)\,\varphi_\mu(t_k).
\end{align}
\end{theorem}
\begin{proof}
By \eqref{eq:W2quantile}, $W_2^2(\mu,\nu)=\norm{(\hat Q_\mu-\hat Q_\nu)+(r_\mu-r_\nu)}^2=D^2+\norm{r_\mu-r_\nu}^2+2\langle\hat Q_\mu-\hat Q_\nu,\,r_\mu-r_\nu\rangle$. Since $\hat Q_\mu$ is constant on each $I_j$ and $r_\mu$ integrates to zero there, $\langle\hat Q_\mu,r_\mu\rangle=0$; likewise $\langle\hat Q_\nu,r_\nu\rangle=0$. Hence the cross term is $-2\langle\hat Q_\mu,r_\nu\rangle-2\langle\hat Q_\nu,r_\mu\rangle$. Abel summation with $\varphi_\nu(0)=0=\varphi_\nu(1)$ gives $\langle\hat Q_\mu,r_\nu\rangle=\sum_{j=1}^M b_j\bigl(\varphi_\nu(s_j)-\varphi_\nu(s_{j-1})\bigr)=-\sum_{j=1}^{M-1}(b_{j+1}-b_j)\varphi_\nu(s_j)$, and symmetrically for $\langle\hat Q_\nu,r_\mu\rangle$.
\end{proof}

\Cref{fig:identity} illustrates the identity. The residual term is nonnegative and at most $(\varepsilon_\mu+\varepsilon_\nu)^2$. The boundary terms have the opposite sign since monotonicity of each quantile implies $\varphi_\mu,\varphi_\nu\le0$, while the atom increments are nonnegative. These terms vanish on a shared partition and account for possible expansion of the sketch distance when partitions differ. Unlike a triangle-inequality bound \citep{beugnot2021,nadjahi2021}, the identity \eqref{eq:identity} distinguishes underestimation from overestimation.

\begin{figure}[t]
\centering
\begin{minipage}[c]{0.5\textwidth}
\centering
\resizebox{\linewidth}{!}{\input{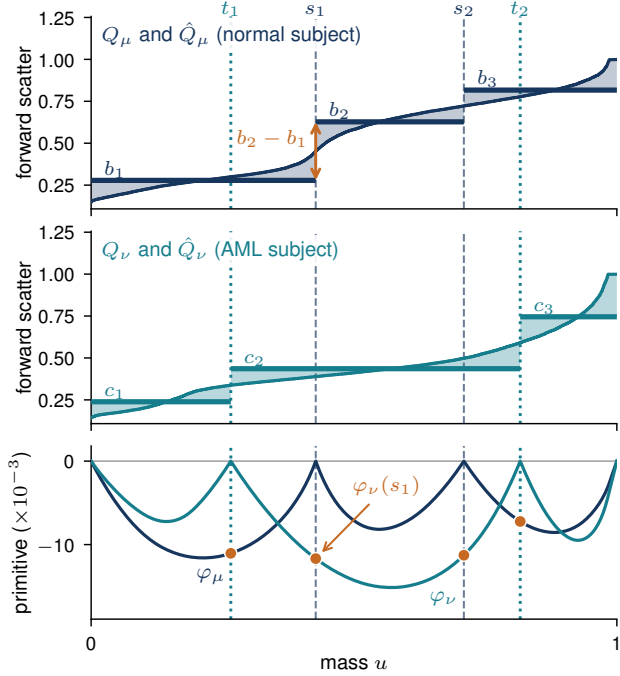}}
\end{minipage}\hfill
\begin{minipage}[c]{0.46\textwidth}
\caption{Partition-mismatch identity \eqref{eq:identity} for forward-scatter projections of a normal and a leukaemia subject in the AML cytometry data ($10^4$ cells each; \cref{sec:experiments}), with optimal $M=3$ sketches. Top and middle: quantiles, step quantiles, and shaded residuals; sketch cells differ. Bottom: residual primitives \eqref{eq:phi}, each zero at its own cell boundaries. Dots pair each sketch’s jumps with the other’s primitive at the same boundary. Here $W_2^2-D^2=-0.0127$, comprising $\norm{r_\mu-r_\nu}^2=0.0085$ and boundary terms $-0.0124$ and $-0.0088$.}
\label{fig:identity}
\end{minipage}
\end{figure}

\begin{corollary}[computable distance certificates]\label{cor:certificates}
Write $S=\varepsilon_\mu+\varepsilon_\nu$, $\Delta_\mu=b_M-b_1$, $\Delta_\nu=c_M-c_1$, and let $q^\mu_{\max},q^\nu_{\max}$ be the largest cell masses. Define
\begin{equation}\label{eq:coarse-boundary-budget}
 H_0=\Delta_\mu\sqrt{q^\nu_{\max}}\,\varepsilon_\nu
    +\Delta_\nu\sqrt{q^\mu_{\max}}\,\varepsilon_\mu.
\end{equation}
If per-cell energies $\varepsilon_{\nu,k}^2=\int_{J_k}r_\nu^2$ are stored, define, for $s\in J_k=(t_{k-1},t_k]$,
\[
 \omega_\nu(s)=\sqrt{\frac{(s-t_{k-1})(t_k-s)}{t_k-t_{k-1}}}\,\varepsilon_{\nu,k},
\]
with value zero at cell boundaries, and define $\omega_\mu$ analogously. The refined boundary budget is
\small{
\begin{equation}\label{eq:cell-boundary-budget} H_{\mathrm{cell}}=2\sum_{j=1}^{M-1}\big((b_{j+1}-b_j)\omega_\nu(s_j)+(c_{j+1}-c_j)\omega_\mu(t_j)\big)\le H_0. \end{equation}}
For $H=H_0$, or $H=H_{\mathrm{cell}}$ when available,
\begin{equation}\label{eq:certB}
 -H\le W_2^2(\mu,\nu)-D^2\le S^2,
\end{equation}
and, writing $x_+=\max\{x,0\}$,
\begin{equation}\label{eq:distance-certificate-interval}
 \max\bigl\{D-S,\sqrt{(D^2-H)_+}\bigr\}
 \le W_2(\mu,\nu)\le\sqrt{D^2+S^2}.
\end{equation}
For unequal cell counts, each sum runs over the interior boundaries of its own sketch.
\end{corollary}
The general budget uses one residual energy per slice; the refinement uses per-cell energies and a linear sweep of both mass partitions. Per-cell energies also support higher-moment correction (\cref{prop:moments}). Lipschitz quantiles with maximum cell mass $O(M^{-1})$ yield $O(M^{-2})$ squared-distance error. \Cref{app:regularity} gives explicit bounds and verifies these conditions for optimal quantizers of densities bounded above and away from zero on compact intervals. The proof is in \cref{app:proofCorollary}.

The boundary budget is not generally $O(S^2)$: \cref{ex:sharp} (\cref{app:sharp}) attains the $M^{-3/2}$ scale allowed by $H_0$ when $q_{\max},\varepsilon_\mu,\varepsilon_\nu\asymp M^{-1}$.

\subsection{Sliced certificates}
\begin{theorem}[sliced certificate]\label{thm:sliced}
Let $\Theta$ be shared, and equip each $\mu_\ell,\nu_\ell$ with centroid sketches having residual errors $\varepsilon_{\mu,\ell},\varepsilon_{\nu,\ell}$. Write $D_{\Theta,M}$ for their comparison and let $H_\ell$ be any corresponding boundary budget from \cref{cor:certificates} or \cref{cor:lipschitz}. With $E_\mu^2=L^{-1}\sum_\ell\varepsilon_{\mu,\ell}^2$ and $E_\nu$ defined analogously,
{\small
\begin{multline}\label{eq:slicedcert}
-\frac1L\sum_{\ell=1}^L H_\ell
\le \SW_{2,\Theta}^2(\mu,\nu)-D_{\Theta,M}^2(\mu,\nu)
\le (E_\mu+E_\nu)^2,\\
\shoveleft{\textnormal{and}~\bigl|\SW_{2,\Theta}(\mu,\nu)
-D_{\Theta,M}(\mu,\nu)\bigr|\le E_\mu+E_\nu.}
\end{multline}}
\end{theorem}
The squared bounds follow by averaging the per-slice certificates and applying Minkowski's inequality. The unsquared bound is the triangle inequality in the product of the slice spaces, as in \citet{nadjahi2021}.

\noindent\textbf{Sampling and direction error.}
If the inputs are empirical measures $\mu^n,\nu^n$, the sketches are built from them, and their comparison $D_{\Theta,M}(\mu^n,\nu^n)$ differs from the population distance $\SW_{2,\Theta}(\mu,\nu)$ for two reasons: the quantization of the empirical slices, and the sampling error of the slices themselves. Let $E_\mu^n,E_\nu^n$ be the residual errors of the sketches of $\mu^n,\nu^n$, and let
\[
 S_\mu^n=\left(\frac1L\sum_{\ell=1}^L
 W_2^2(\mu_\ell^n,\mu_\ell)\right)^{1/2}
\]
be the sliced distance between the empirical and population projections of $\mu$, with $S_\nu^n$ defined similarly. \Cref{thm:sliced}, applied to the empirical inputs, bounds the quantization error by $E_\mu^n+E_\nu^n$, and the triangle inequality for $\SW_{2,\Theta}$ bounds the sampling error by $S_\mu^n+S_\nu^n$:
\[
 \abs{D_{\Theta,M}(\mu^n,\nu^n)-\SW_{2,\Theta}(\mu,\nu)}
 \le E_\mu^n+E_\nu^n+S_\mu^n+S_\nu^n.
\]
The first two terms are read off the sketches; the last two are not observable and depend on the sample sizes and the data. This deterministic decomposition therefore leaves sampling rates to the assumptions on the data \citep{nadjahi2020,nietert2022}. For independent uniformly sampled directions, put $Z(\theta)=W_2^2(\theta_\#\mu,\theta_\#\nu)$. The Monte Carlo estimate satisfies \citep[Theorem~6]{nadjahi2020}
\[
 \E_\Theta\abs{\SW_{2,\Theta}^2-\SW_2^2}
 \le\sqrt{\Var_\theta(Z)/L}.
\]
Finite second moments suffice here because $Z$ is bounded uniformly in $\theta$. The corresponding expected unsquared error is also $O(L^{-1/2})$, with a moment-dependent constant; details are in \cref{app:proofs}.

\subsection{Rates: free weights vs. equal-mass cells}
\begin{theorem}[sliced quantization complexity]\label{thm:rate}
If $X\sim\mu$ satisfies $\E\norm{X}^{2+\delta}<\infty$ for some $\delta>0$, then for the optimal-quantizer sketch
\begin{equation}\label{eq:pierce}
 \sup_{\theta\in\Sph}\varepsilon_M^\star(\theta_\#\mu)\le \frac{C_\delta\,\bigl(\E\norm{X}^{2+\delta}\bigr)^{1/(2+\delta)}}{M},
\end{equation}
hence $E_\mu\le C_\delta(\E\norm{X}^{2+\delta})^{1/(2+\delta)}/M$
for every $\Theta$, for a universal constant $C_\delta$.%
\end{theorem}

See \cref{prop:equalmass} in \cref{app:equalmass} for more details. %

\subsection{Moments}
A sketch is also a summary of each projected distribution, so a natural question is how well its directional moments, and the moment tensors of $X$ assembled from them across directions, match those of the input. The next proposition shows that the mean of every slice is exact, that the second-moment deficit equals the stored residual energy, and that the error in higher moments is controlled by that energy, with a sharper bound when per-cell energies are stored; combining the slices by least squares then bounds the error of the recovered moment tensors.
\begin{proposition}[moment recovery]\label{prop:moments}
Let $X\sim\mu\in\Pcal_2(\R^d)$ and, for a unit direction $\theta$, couple $X_\theta=\theta^\top X$ in law with its centroid sketch $\hat X_\theta$ through their quantiles. Write $\varepsilon_{\mu,\theta}^2=\E(X_\theta-\hat X_\theta)^2$; on the chosen directions, $\varepsilon_{\mu,\theta_\ell}=\varepsilon_{\mu,\ell}$ as in \cref{thm:sliced}. Then
\begin{equation}\label{eq:centroid-moment-identities}
 \E\hat X_\theta=\E X_\theta,\qquad
 \E X_\theta^2=\E\hat X_\theta^2+\varepsilon_{\mu,\theta}^2,
\end{equation}
and $\abs{\E f(X_\theta)-\E f(\hat X_\theta)}\le\tfrac12\norm{f''}_\infty\varepsilon_{\mu,\theta}^2$ for $f\in C^2(\R)$ with bounded second derivative.
If $\norm X\le R$ almost surely, with $R>0$, then for every integer $r\ge2$, writing $m_{r,\theta}=\E X_\theta^r$ and $\hat m_{r,\theta}=\E\hat X_\theta^r$,
\begin{equation}\label{eq:uncorrected-moment-bound}
 \abs{m_{r,\theta}-\hat m_{r,\theta}}
 \le\binom r2R^{r-2}\varepsilon_{\mu,\theta}^2,
\end{equation}
with $m_{r,\theta}\ge\hat m_{r,\theta}$ for even $r$.
For quantile cells $I_{\theta,j}$ with centroids $b_{\theta,j}$, let
\[
 e_{\theta,j}=\int_{I_{\theta,j}}\bigl(Q_{\theta_\#\mu}(x)-b_{\theta,j}\bigr)^2\,dx.
\]
If these per-cell energies are stored, define
\begin{equation}\label{eq:corrected-moment-estimate}
 \tilde m_{r,\theta}
 =\hat m_{r,\theta}+\binom r2\sum_j b_{\theta,j}^{r-2}e_{\theta,j}.
\end{equation}
Then $\tilde m_{2,\theta}=m_{2,\theta}$ and, whenever $\abs{X_\theta-\hat X_\theta}\le h_\theta$ almost surely, for $r\ge3$,
\begin{equation}\label{eq:corrected-moment-bound}
 \abs{m_{r,\theta}-\tilde m_{r,\theta}}
 \le\binom r3R^{r-3}h_\theta\varepsilon_{\mu,\theta}^2.
\end{equation}
For $T_r=\E X^{\otimes r}$, let $A_rT=(\langle T,\theta_\ell^{\otimes r}\rangle)_{\ell=1}^L$ on symmetric $r$-tensors. If $A_r$ has full column rank, the least-squares reconstructions from the uncorrected and corrected moment vectors satisfy
\begin{align}
 \norm{\hat T_r-T_r}_F
 &\le\binom r2R^{r-2}\norm{A_r^\dagger}_{2\to F}
 \bigl\|(\varepsilon_{\mu,\ell}^2)_\ell\bigr\|_2,
 &&r\ge2,\label{eq:uncorrected-tensor-bound}\\
 \norm{\tilde T_r-T_r}_F
 &\le\binom r3R^{r-3}\norm{A_r^\dagger}_{2\to F}
 \bigl\|(h_{\theta_\ell}\varepsilon_{\mu,\ell}^2)_\ell\bigr\|_2,
 &&r\ge3.\label{eq:corrected-tensor-bound}
\end{align}
Full rank holds for $L\ge\binom{d+r-1}{r}$ generic directions \citep{bodmann2018}. Without bounded support, the mean is still recovered exactly from $d$ spanning directions, and storing one total residual energy per slice recovers the covariance exactly from $\binom{d+1}{2}$ slices.
\end{proposition}
The centroid identities and the $C^2$ weak-error bound are classical \citep{pages2003,lemaire2020}. For every fixed order, $\varepsilon_{\mu,\theta}=O(M^{-1})$ gives $O(M^{-2})$ moment error; the correction improves this if also $h_\theta=O(M^{-1})$, as under the bounded-density hypotheses of \cref{lem:cellwidth}. Per-cell energies are the metadata used in \cref{cor:certificates}; no higher residual moments are required. Combining the identities with least squares gives a deterministic error bound for the sliced-moment construction also used by \citet{liu2025emperor}. Exactness concerns the measure being sketched, hence empirical moments for empirical inputs.

\section{Experiments}\label{sec:experiments}

\begin{table}[t]
\centering\small
\caption{Classification with $L=16$ shared learned slices. $N$: points per distribution (mean distinct words for documents); $M$: smallest optimal-sketch size matching uncompressed validation accuracy within $0.5$ percentage points, also used for equal mass sketches. Scores are percentages. Bold marks the best method, and $^{*}$ a significant paired difference between ours and equal-mass.
Error ratio: median equal mass/optimal quantization energy. (error ratio > 1 shows that equal-mass results in larger error compared to ours.)
}
\label{tab:compression}
{\begin{tabular}{@{}lccccccc@{}}
\toprule
benchmark & $(N,M)$ & BW & Exact OT & unc.\ $\SW_{2,\Theta}$ & equal-mass & ours & error ratio\\
\midrule
 & & \multicolumn{5}{c}{test error ($\downarrow$)} & \\
\cmidrule(lr){3-7}
BBCSport & $(117,4)$ & $6.2\pm1.9$ & $5.1\pm1.2$ & $1.1\pm0.8$ & $1.9\pm1.1$ & $\mathbf{1.1\pm0.8}$ & $1.8$\\
Twitter & $(10,4)$ & $28.7\pm1.4$ & $29.7\pm1.5$ & $27.0\pm1.4$ & $27.0\pm0.9$ & $\mathbf{26.7\pm0.2}$ & $4.0$\\
Classic & $(39,4)$ & $3.2\pm0.5$ & $\mathbf{2.9\pm0.4}$ & $3.1\pm0.3$ & $3.8\pm0.4$ & $3.3\pm0.3^{*}$ & $1.8$\\
Recipe & $(48,8)$ & $43.2\pm1.0$ & $42.9\pm0.8$ & $37.4\pm0.9$ & $38.7\pm0.7$ & $\mathbf{37.4\pm1.0}^{*}$ & $4.1$\\
Reuters & $(36,2)$ & $4.7$ & $3.9$ & $2.9$ & $3.4$ & $\mathbf{2.8}$ & $1.2$\\
\midrule
 & & \multicolumn{5}{c}{accuracy ($\uparrow$)} & \\
\cmidrule(lr){3-7}
AML & $(10^4,2)$ & $97.5\pm1.4$ & $97.4\pm1.2$ & $98.1\pm1.5$ & $97.7\pm1.7$ & $\mathbf{98.2\pm1.5}^{*}$ & $1.3$\\
\midrule
ModelNet40 & $(1024,8)$ & $62.4$ & $\mathbf{83.2}$ & $80.1$ & $79.0$ & $80.3^{*}$ & $1.3$\\
\bottomrule
\end{tabular}}
\end{table}

\noindent\textbf{Classification.}
We evaluate distribution classification on five document datasets
\citep{kusner2015,sato2022}, AML patient samples ~\citep{aghaeepour2013} represented by cell
measurements, and ModelNet40~\citep{mn40} point clouds (\cref{tab:compression}).
All sliced methods use the same $L=16$ linear projections, learned
 and then frozen. We compare the uncompressed
sliced distance, optimal sketches, and equal-mass sketches with the
same number of atoms $M$. Equal-mass sketches use $M$ quantile intervals of fixed mass $1/M$, while ours adapt the interval masses to minimize
quantization error. Classification uses nearest neighbours on each distance. For each benchmark, we select the smallest
$M$ for which the optimal sketch is within $0.5$ percentage points
of the uncompressed validation accuracy. Optimal and equal-mass
sketches store $L(2M-1)$ and $LM$ values, respectively, compared
with $LN$ for uncompressed projections. We also report exact optimal transport with Euclidean ground cost (Word Mover's Distance for documents) and the Bures--Wasserstein distance (BW) between Gaussian approximations defined by each distribution's mean and covariance.
These baselines are unsupervised, whereas our projections use
training labels. 
We provide implementation details for all experiments in \cref{app:impl}, and ablations in  \cref{app:sweeps}.

Optimal sketches store only $48$--$240$ values per distribution
while remaining within $0.3$ percentage points of the uncompressed
test score on all seven benchmarks. This corresponds to storage
reductions of $1.5$--$18\times$ for documents, $68\times$ for point
clouds, and over $3000\times$ for AML samples. At the same $M$,
optimal sketches outperform equal-mass sketches on all seven
benchmarks, with statistically significant differences on Classic,
Recipe, AML, and ModelNet40. Optimal sketches also outperform the exact transport baselines on five benchmarks. Sketches are also more accurate than BW on six benchmarks. 

\noindent\textbf{Learning from few subjects.}
Sketch embeddings also support classification when labelled distributions are scarce. On AML, each subject is represented on $L=32$ random slices of its $N=10^4$ cells, either by the $LN=320{,}000$ uncompressed sorted projections or by a sketch: the $M$ atoms and $M-1$ masses of every slice for the optimal sketch, and $2M-1$ cell means, the same number, for equal-mass cells. We report results for $M=2$ and $M=8$, resulting in stored values $96$ and $480$, respectively. As a baseline, we concatenate each subject's empirical mean and the unique entries of its covariance matrix in the original $d=7$ cell-feature space. Inside each of the $25$ training folds of \cref{tab:compression} we draw $2$, $4$, $8$ or $16$ AML subjects with normal subjects in the ratio of the dataset, ten times, and fit class-balanced $\ell_2$-regularized logistic regression. \Cref{tab:amlmain_logreg} reports the test AUC, and a linear SVM leads to the same conclusions (\cref{app:svm}).
Every representation reaches a training AUC of $1$, while their test AUCs reveal differences in generalization to held-out subjects.

\begin{table}[t]
\centering\small
\caption{\textbf{AML, logistic regression from few labelled subjects.} Test AUC (\%), mean $\pm$ standard error over the $25$ folds, with ten draws of the training subjects per fold. $L=32$ random slices (seed $0$), $N=10^4$ cells. Normal subjects are added in the ratio of the dataset; ``all'' is the whole training fold, about $34$ AML subjects. Stored values: number of values stored per subject. Linear SVM in \cref{tab:amlmain_svm}.%
}
\label{tab:amlmain_logreg}
{\begin{tabular}{@{}lcccccc@{}}
\toprule
 & & \multicolumn{5}{c}{AML subjects in training}\\
\cmidrule(l){3-7}
summary & stored values & 2 & 4 & 8 & 16 & all\\
\midrule
uncompressed & 320{,}000 & $92.7\pm0.6$ & $95.4\pm0.4$ & $97.1\pm0.4$ & $98.1\pm0.4$ & $98.6\pm0.4$\\
mean, covariance & 35 & $92.4\pm0.8$ & $96.2\pm0.5$ & $97.1\pm0.6$ & $97.9\pm0.5$ & $98.1\pm0.5$\\
\midrule
equal-mass & 96 & $92.4\pm0.6$ & $95.1\pm0.5$ & $96.7\pm0.5$ & $97.6\pm0.4$ & $97.9\pm0.5$\\
ours & 96 & $\mathbf{96.8\pm0.4}$ & $\mathbf{98.1\pm0.2}$ & $\mathbf{98.8\pm0.2}$ & $\mathbf{99.3\pm0.2}$ & $\mathbf{99.6\pm0.1}$\\
\addlinespace
equal-mass & 480 & $92.9\pm0.6$ & $95.6\pm0.5$ & $97.3\pm0.4$ & $98.2\pm0.4$ & $98.5\pm0.4$\\
ours & 480 & $\mathbf{96.6\pm0.4}$ & $\mathbf{97.9\pm0.3}$ & $\mathbf{98.6\pm0.3}$ & $\mathbf{99.1\pm0.3}$ & $\mathbf{99.2\pm0.3}$\\
\bottomrule
\end{tabular}}
\end{table}

The optimal sketch has a higher test AUC than the uncompressed
representation at every training size, demonstrating an advantage when labelled subjects are scarce. The sketch provides a compact, adaptive representation that generalizes better with the same linear classifier. The improvement is not explained by feature count alone. Equal-mass cells at matched storage remain within $1.2$ points of the uncompressed representation, and mean and covariance features do not consistently improve on it. These results suggest that the adaptive representation contributes to the observed generalization gains beyond dimensionality reduction alone.

\noindent\textbf{Sketches as a pooling layer.}\label{sec:pooling}
\begin{figure}[!t]
\centering
\begin{minipage}[c]{0.5\textwidth}
\centering
\resizebox{\linewidth}{!}{\input{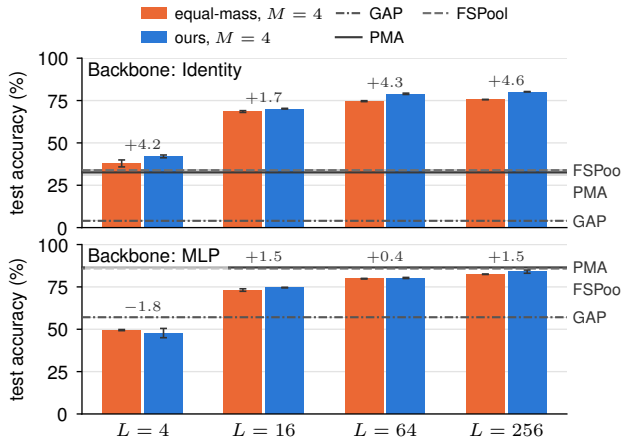}}
\end{minipage}\hfill
\begin{minipage}[c]{0.46\textwidth}
\caption{ModelNet40, following  \citep{naderializadeh2021}: no backbone (top), and MLP backbone (bottom). Test accuracy (\%), mean and standard deviation (error bars) over $3$ seeds. Equal-mass and ours: sketches with $M=4$ cells per slice; the number above each pair is ours minus equal-mass.
Lines: GAP, PMA and FSPool, which use no slices, with $\pm$ one standard deviation shaded.}
\label{fig:pooling}
\end{minipage}
\end{figure}
The sketch can also serve as the pooling layer of a network trained end to end. We follow the ModelNet40 protocol and official code of \citet{naderializadeh2021}. We use a per-point MLP, or no backbone (Identity)\citep{deepset}, $L$ learned linear slices of its output, a pooling layer, and a linear classifier. Their PSWE layer is an equal-mass summary at full resolution ($1024$ cells per slice) followed by learned rank weights that reduce each slice to one number. We apply the same layer, with the same reference and rank weights, to a sketch with $4$ cells per slice, either equal-mass cells or the optimal sketch, so the two sketches differ only in their partition. For $M=N$ both sketches reproduce PSWE. We therefore only compare the two sketches at $M=4$. The optimal sketch is computed exactly in every forward pass; gradients flow through its atoms, which are cell means, while the partition is held fixed.

Without a backbone the slices act on the raw coordinates, and the optimal sketch is more accurate than equal-mass cells at every $L$, by $1.7$ to $4.6$ points (\cref{fig:pooling}). With the MLP backbone it is more accurate for $L\ge16$, by $0.4$ to $1.5$ points, and less accurate at $L=4$, where it varies most across seeds. With the MLP backbone, the optimal sketch comes within $2.5$ points of PMA and $1.8$ points of FSPool \citep{zhang2020fspool} at $L=256$, against $4.0$ and $3.4$ points for equal-mass cells, while representing each projected distribution with four cells. Without a backbone, these layers output only three numbers per cloud and reach at most $34\%$, while the optimal sketch reaches $80\%$ at $L=256$.

\section{Conclusion}\label{sec:conclusion}
We studied compact sliced-Wasserstein sketches via optimal one-dimensional quantization. Our partition-mismatch identity characterizes how independently chosen cells affect pairwise distances and yields error certificates for centroid sketches. Experiments show that optimal sketches can preserve uncompressed sliced-distance classification performance under substantial compression. Minimizing quantization error does not guarantee the best distance approximation or downstream performance, and adaptive sketches require more storage than equal-mass sketches at the same atom count. Future work targets faster construction and sharper distance-error guarantees for optimal quantizers.

\acknowledgments
This work was supported in part by NSF CAREER Award No. 233989 and NSF DMS Award No. 2603773. The authors also gratefully acknowledge the computing resources provided by the NVIDIA Academic Grant Program.

\bibliographystyle{plainnat}
\bibliography{references,references_extra}

\clearpage
\appendix

\section{Proofs}\label{app:proofs}
\subsection{Proof of \cref{cor:certificates}}\label{app:proofCorollary}
Write the boundary term of \eqref{eq:identity} as
\[
 B=2\sum_j(b_{j+1}-b_j)\varphi_\nu(s_j)+2\sum_k(c_{k+1}-c_k)\varphi_\mu(t_k),
\]
so that $W_2^2(\mu,\nu)-D^2=\norm{r_\mu-r_\nu}^2+B$. For a cell $J_k=(a,b]$ with centroid $c_k$, the zero-cell-mean condition gives $\varphi_\nu(s)=\int_a^s(Q_\nu(x)-c_k)\,dx$ for $s\in[a,b]$. For $s\in(a,b)$, the average of a nondecreasing function over $(a,s]$ is no larger than its average over $(a,b]$, so $\varphi_\nu(s)\le0$; the primitive is zero at both endpoints. The same holds for $\varphi_\mu$. Since all atom increments are nonnegative, $B\le0$.

Using $\int_a^b r_\nu=0$, write
\[
 \varphi_\nu(s)=\int_a^b r_\nu(x)
 \left(\mathbf1_{(a,s]}(x)-\frac{s-a}{b-a}\right)dx.
\]
The squared $L^2(a,b)$ norm of the function in parentheses is $(s-a)(b-s)/(b-a)$. Cauchy--Schwarz therefore gives
\[
 \abs{\varphi_\nu(s)}\le
 \sqrt{\frac{(s-a)(b-s)}{b-a}}\,\varepsilon_{\nu,k}=\omega_\nu(s)
 \le\tfrac12\sqrt{q^\nu_{\max}}\,\varepsilon_\nu,
\]
and the analogous bounds hold for $\mu$. Summing against the nonnegative atom increments, whose sums are $\Delta_\mu$ and $\Delta_\nu$, yields $-H_{\mathrm{cell}}\le B\le0$ and $H_{\mathrm{cell}}\le H_0$. Together with $0\le\norm{r_\mu-r_\nu}^2\le S^2$, this proves \eqref{eq:certB} for either budget. Taking square roots yields $\sqrt{(D^2-H)_+}\le W_2(\mu,\nu)\le\sqrt{D^2+S^2}$, and the triangle inequality gives $W_2(\mu,\nu)\ge D-S$, proving \eqref{eq:distance-certificate-interval}. To evaluate $H_{\mathrm{cell}}$, each boundary $s_j$ of $\hat\mu$ must be located in
its containing cell $J_{k(j)}$ of $\hat\nu$. Since both boundary lists are sorted,
$k(j)$ is nondecreasing in $j$, so a single merge of the two lists finds every
$k(j)$; the same merge with the roles of $\mu$ and $\nu$ exchanged locates every
$t_k$ in a cell of $\hat\mu$. Hence $H_{\mathrm{cell}}$ costs $O(M_\mu+M_\nu)$ time. \qed %

\subsection{Regularity refinements of \cref{cor:certificates}}\label{app:regularity}

Here we illustrate how \cref{cor:certificates} can be strengthened when more assumptions are made on the quantiles. While this is not directly applicable to discrete measures, it provides some intuition for the more general approximation scheme considered here.

\begin{corollary}[Lipschitz quantiles]\label{cor:lipschitz}
Suppose $Q_\mu,Q_\nu$ are Lipschitz with constants $A_\mu,A_\nu$. Then the boundary budget
\[
 H_A=\tfrac14\Delta_\mu A_\nu(q^\nu_{\max})^2
    +\tfrac14\Delta_\nu A_\mu(q^\mu_{\max})^2
\]
may replace $H$ in \eqref{eq:certB} and \eqref{eq:distance-certificate-interval}. Furthermore, $\varepsilon_\rho^2\le A_\rho^2(q^\rho_{\max})^2/12$ for $\rho\in\{\mu,\nu\}$, and consequently
\[
 \abs{W_2^2(\mu,\nu)-D^2}
 \le\max\Bigl\{H_A,\frac{(A_\mu q^\mu_{\max}+A_\nu q^\nu_{\max})^2}{12}\Bigr\}.
\]
For fixed Lipschitz constants, this is $O(M^{-2})$ whenever both largest cell masses are $O(M^{-1})$.
\end{corollary}
\begin{proof}
Write $B$ for the sum of the two boundary terms in \eqref{eq:identity}, so that
$W_2^2(\mu,\nu)-D^2=\norm{r_\mu-r_\nu}^2+B$. On a cell $J_k=(a,b]$, the centroid formula gives
\[
 -\varphi_\nu(s)=\frac1{b-a}\int_a^s\int_s^b\bigl(Q_\nu(y)-Q_\nu(x)\bigr)\,dy\,dx
 \le\frac{A_\nu}{2}(s-a)(b-s)\le\frac{A_\nu(b-a)^2}{8}.
\]
Since $\varphi_\nu\le0$ and the atom increments $b_{j+1}-b_j$ are nonnegative with sum
$\Delta_\mu$,
\[
 0\le-2\sum_{j}(b_{j+1}-b_j)\,\varphi_\nu(s_j)
 \le2\cdot\frac{A_\nu(q^\nu_{\max})^2}{8}\sum_{j}(b_{j+1}-b_j)
 =\tfrac14\Delta_\mu A_\nu(q^\nu_{\max})^2 .
\]
The same argument with the roles of $\mu$ and $\nu$ exchanged bounds the second sum in
$B$ by $\tfrac14\Delta_\nu A_\mu(q^\mu_{\max})^2$. Adding the two bounds gives
$-H_A\le B\le0$, so the proof of \cref{cor:certificates} applies with $H_A$ in place of
$H$. For the energy, the variance identity on a cell of mass $q_k=b-a$ gives
\[
 \varepsilon_{\nu,k}^2=\frac1{2q_k}\int_a^b\!\!\int_a^b\bigl(Q_\nu(x)-Q_\nu(y)\bigr)^2\,dx\,dy
 \le\frac{A_\nu^2q_k^3}{12}.
\]
Summing over cells and using $\sum_kq_k=1$ yields $\varepsilon_\nu^2\le A_\nu^2(q^\nu_{\max})^2/12$; the argument for $\mu$ is identical. The absolute bound follows from \eqref{eq:certB} and the resulting bound on $S^2$.
\end{proof}

\begin{lemma}[widths of optimal quantization cells]\label{lem:cellwidth}
Let $\rho$ have density $f$ with $0<c\le f\le C$ almost everywhere on an interval of length $\ell$, vanishing outside it. For $M\ge6$, every optimal free-weight $M$-atom quadratic quantizer has clipped Voronoi cells of width at most $\kappa\ell/M$, where $\kappa=4^{4/3}(C/c)^{4/3}$. Thus
\[
 q_{\max}\le\frac{C\kappa\ell}{M},\qquad
 \abs{Y-\hat Y}\le\frac{\kappa\ell}{M}\ \text{almost surely},\qquad
 \varepsilon_\rho\le\frac{\ell}{2M},
\]
where $(Y,\hat Y)$ is the quantile coupling of $\rho$ with its sketch. Also, $Q_\rho$ is $c^{-1}$-Lipschitz.
\end{lemma}
\begin{proof}
An optimum uses exactly $M$ nonempty cells: otherwise, splitting any positive-width cell into two and replacing its centroid by the two conditional centroids strictly reduces the error while using at most $M$ atoms. The cells partition the support interval, up to endpoints, and each atom is its cell centroid.

Fix a cell $I$ of width $w$ and mass $q\le Cw$. Split it at its midpoint into halves $I_1,I_2$ with masses $q_i\ge cw/2$ and centroids $m_1<m_2$. On a half of length $h=w/2$, the centroid's distance from the inner endpoint is at least $c\int_0^ht\,dt/(Ch)=ch/(2C)$, so $m_2-m_1\ge cw/(2C)$. By the decomposition $\int_I(x-m)^2\,d\rho=\sum_i\int_{I_i}(x-m_i)^2\,d\rho
+\frac{q_1q_2}{q}(m_2-m_1)^2$, where $m$ is the centroid of $I$, replacing $I$ by its two halves decreases the distortion by
\[
 \frac{q_1q_2}{q}(m_2-m_1)^2\ge\frac{c^4w^3}{16C^3}.
\]
Excluding $I$, the remaining $M-1$ cells form at most two chains and contain at least $(M-3)/2$ disjoint adjacent pairs. Their widths sum to at most $\ell$, so one pair has combined width $W\le2\ell/(M-3)\le4\ell/M$, using $M\ge6$. Merging that pair into one cell at its centroid increases the distortion by $\frac{q_jq_{j+1}}{q_j+q_{j+1}}(b_{j+1}-b_j)^2\le CW^3/4$. Splitting $I$ and merging this disjoint pair gives a competitor with $M$ cells, so optimality forces $c^4w^3/(16C^3)\le CW^3/4$, hence $w\le4^{1/3}(C/c)^{4/3}W\le\kappa\ell/M$. Its mass is at most $Cw$, proving the mass bound. As the centroid lies in its cell, the same width bound controls $\abs{Y-\hat Y}$ almost surely. Finally, $M$ equally spaced midpoint representatives on the support interval have pointwise nearest-neighbour error at most $\ell/(2M)$; assigning each point to its nearest representative gives an admissible free-weight quantizer, so optimality gives $\varepsilon_\rho\le\ell/(2M)$. The lower density bound gives $F(y)-F(x)\ge c(y-x)$ on the support interval, and inverting gives the Lipschitz bound on $Q_\rho$.
\end{proof}

\begin{corollary}[optimal sketches of bounded densities]\label{cor:density}
For $\rho\in\{\mu,\nu\}$, suppose $\rho$ has density $0<c_\rho\le f_\rho\le C_\rho$ almost everywhere on an interval of length $\ell_\rho$, vanishing outside it. Put $\kappa_\rho=4^{4/3}(C_\rho/c_\rho)^{4/3}$ and $\beta_\rho=C_\rho\kappa_\rho\ell_\rho$. If both sketches are optimal free-weight quantizers with $M\ge6$ atoms, then $\abs{W_2^2(\mu,\nu)-D^2}\le K_{\mu\nu}/M^2$, where
\[
 K_{\mu\nu}=\max\Bigl\{
 \tfrac14\Bigl(\frac{\ell_\mu\beta_\nu^2}{c_\nu}+\frac{\ell_\nu\beta_\mu^2}{c_\mu}\Bigr),
 \frac{(\ell_\mu+\ell_\nu)^2}{4}\Bigr\}.
\]
\end{corollary}
\begin{proof}
\Cref{lem:cellwidth} gives $q^\rho_{\max}\le\beta_\rho/M$, a Lipschitz constant $A_\rho=c_\rho^{-1}$ for $Q_\rho$, and $\varepsilon_\rho\le\ell_\rho/(2M)$. Since every centroid is in the support interval, $\Delta_\rho\le\ell_\rho$. Thus $H_A\le\frac1{4M^2}\bigl(\ell_\mu\beta_\nu^2/c_\nu+\ell_\nu\beta_\mu^2/c_\mu\bigr)$ and $S^2\le(\ell_\mu+\ell_\nu)^2/(4M^2)$, and \eqref{eq:certB} with the budget of \cref{cor:lipschitz} gives the result.
\end{proof}

\subsection{Sharpness of the boundary budget}\label{app:sharp}
The following example shows that the boundary budget of \cref{cor:certificates} cannot in general be replaced by a constant multiple of $S^2$.
\begin{example}[sharpness for general centroid sketches]\label{ex:sharp}
Let $\mu=\tfrac12U[-2,-1]+\tfrac12U[1,2]$ and use its equal-mass sketch with $M$ even. The middle boundary is $s_{M/2}=\tfrac12$, with atom jump $b_{M/2+1}-b_{M/2}=2+2/M$. Set $q=1/(M+1)$ and
\[
 \nu_M=(1-q)\mu+\tfrac q2\delta_{-h/2}+\tfrac q2\delta_{h/2},
 \qquad h=M^{-1/2}.
\]
For the equal-mass sketch of $\nu_M$ with $M+1$ cells, the middle cell consists exactly of the two added atoms. Its centroid is zero, so $\varphi_{\nu_M}(\tfrac12)=-qh/4$. The other boundary terms together are $O(M^{-2})$: away from the central jump the quantiles are linear with bounded slopes, while the two central boundaries of the second sketch lie in cells of $\mu$ with residual integrals $O(M^{-2})$. The residual term is also $O(M^{-2})$. Thus
\begin{align*}
 W_2^2(\mu,\nu_M)-D^2
 &=-\frac{(2+2/M)qh}{2}+O(M^{-2})\\
 &=-M^{-3/2}+O(M^{-2}),
\end{align*}
whereas $\varepsilon_\mu^2=1/(3M^2)$ and $\varepsilon_{\nu_M}^2=qh^2/4+O(M^{-2})$. This is a sequence of pairs with general centroid sketches, checked numerically in \cref{app:theory-validation}. Whether an $O(M^{-2})$ gap holds for all pairs of optimal sketches with $\varepsilon_\mu,\varepsilon_\nu=O(M^{-1})$ remains open.
\end{example}

\subsection{Proof of \cref{thm:sliced}}
Average \eqref{eq:certB} over the slices. Minkowski's inequality gives
\[
 \frac1L\sum_\ell(\varepsilon_{\mu,\ell}+\varepsilon_{\nu,\ell})^2
 \le(E_\mu+E_\nu)^2,
\]
which proves the signed squared bounds. For the unsquared bound, apply the reverse triangle inequality to the Euclidean norms of the vectors of slice distances, and then use the fact that
$\abs{W_2(\mu_\ell,\nu_\ell)-W_2(\hat\mu_\ell,\hat\nu_\ell)}
\le\varepsilon_{\mu,\ell}+\varepsilon_{\nu,\ell}$.

For empirical inputs, the same product-space triangle inequality gives
\[
 \abs{\SW_{2,\Theta}(\mu,\nu)-D_{\Theta,M}(\mu^n,\nu^n)}
 \le S_\mu^n+S_\nu^n+E_\mu^n+E_\nu^n.
\]
This holds for any empirical measures; independence and sampling assumptions enter only when bounding $S_\mu^n,S_\nu^n$ statistically.

For the Monte Carlo statement, put $m=\E_\theta Z=\SW_2^2$ and $\bar Z=L^{-1}\sum_\ell Z(\theta_\ell)=\SW_{2,\Theta}^2$. Independence gives $\E_\Theta(\bar Z-m)^2=\Var_\theta(Z)/L$. Also
\[
 0\le Z(\theta)\le K:=2\bigl(\E\norm X^2+\E\norm Y^2\bigr),
 \qquad X\sim\mu,\quad Y\sim\nu,
\]
so $\Var(Z)\le Km$. Cauchy--Schwarz gives the squared-error claim. If $m>0$,
\[
 \E_\Theta\abs{\sqrt{\bar Z}-\sqrt m}
 \le\frac{\E_\Theta\abs{\bar Z-m}}{\sqrt m}
 \le\sqrt{K/L}.
\]
If $m=0$, $Z=0$ almost surely and the error is zero. These are the usual projection-sampling bounds \citep{nadjahi2020}. \qed

\subsection{Proof of \cref{thm:rate}}
Fix $\theta\in\Sph$ and let $X_\theta=\theta^\top X$. The extended Pierce lemma \citep[Lemma~1]{luschgy2008} states that for every $\delta>0$ there is a constant $C_\delta$ such that for every real random variable $Y$ and every $M\ge1$, $\inf_{\abs{\alpha}\le M}\norm{Y-\hat Y^\alpha}_{L^2}\le C_\delta\norm{Y}_{L^{2+\delta}}M^{-1}$, where $\hat Y^\alpha$ is the nearest-neighbour projection onto the finite set $\alpha$; the left side is $\varepsilon_M^\star(\operatorname{law}(Y))$. Apply it to $Y=X_\theta$ and use $\abs{\theta^\top X}\le\norm X$, so $\norm{X_\theta}_{L^{2+\delta}}\le(\E\norm X^{2+\delta})^{1/(2+\delta)}$ uniformly in $\theta$. Averaging $\bigl(\varepsilon_M^\star(\theta_{\ell\#}\mu)\bigr)^2$ over $\ell$ gives the bound on $E_\mu$. (Translating $X$ by any fixed vector does not change $\varepsilon_M^\star$, so $\E\norm X^{2+\delta}$ may be replaced by $\min_a\E\norm{X-a}^{2+\delta}$; see \citealp[Thm.~1.9(b)]{lemaire2020}.) The limit statement is Zador's theorem in one dimension \citep[Thm.~6.2]{grafLuschgy2000}, valid under the same $2+\delta$ moment condition, which also guarantees $\int f_\theta^{1/3}<\infty$. Note that a finite second moment alone does not guarantee the stated rate \citep{grafLuschgy2000}. \qed

\subsection{Equal-mass approximation rates}\label{app:equalmass}

The following proposition compares the equal-mass sketch with the optimal quantizer
of a single one-dimensional law $\rho$, such as a projection $\theta_\#\mu$. Write
$\varepsilon_M^{\mathrm{eq}}(\rho)$ for the error of its $M$-cell equal-mass sketch and
$\varepsilon_M^\star(\rho)$ for the optimal $M$-atom error. Since equal weights are
admissible, $\varepsilon_M^\star(\rho)\le\varepsilon_M^{\mathrm{eq}}(\rho)$ for every $M$.

\begin{proposition}[equal-mass and optimal rates]\label{prop:equalmass}
Let $\rho$ have density $f$ and an absolutely continuous quantile function.
\begin{enumerate}[label=(\alph*),leftmargin=*]
\item If $\int_{\{f>0\}}f^{-1}<\infty$, then
 $M^2\bigl(\varepsilon_M^{\mathrm{eq}}\bigr)^2\to\frac1{12}\int_{\{f>0\}}f^{-1}$.
\item If, in addition, $\rho$ has a finite moment of order $2+\delta$ for some
 $\delta>0$, then
 $\lim M^2(\varepsilon_M^\star)^2=\frac1{12}\bigl(\int f^{1/3}\bigr)^3
 \le\frac1{12}\int_{\{f>0\}}f^{-1}$, with equality if and only if $\rho$ is uniform on
 an interval.
\item If $\rho$ is a nondegenerate Gaussian, then
 $(\varepsilon_M^{\mathrm{eq}})^2\asymp(M\log M)^{-1}$, whereas
 $(\varepsilon_M^\star)^2\asymp M^{-2}$.
\item If $\rho$ is the projection of the uniform distribution on the unit ball of
 $\R^d$ onto a line, so that $f(t)\propto(1-t^2)^{(d-1)/2}$ on $[-1,1]$, then
 $(\varepsilon_M^\star)^2\asymp M^{-2}$ and
 \[
 (\varepsilon_M^{\mathrm{eq}})^2\asymp
 \begin{cases}
 M^{-2},&d\le2,\\
 M^{-2}\log M,&d=3,\\
 M^{-(d+5)/(d+1)},&d\ge4,
 \end{cases}
 \]
 with constants depending on $d$.
\end{enumerate}
\end{proposition}
In (c), and in (d) for $d\ge3$, $\int f^{-1}=\infty$, so (a) does not apply and the
equal-mass rate is slower than $M^{-2}$. The cost comes from low-density cells: in
the Gaussian tail and near the endpoints of the ball projection, a cell of mass $1/M$
must be wide. The Gaussian rate in (c) is sharp along $M=2^k$ by
\citet[Theorem~1]{giles2019}.

\begin{proof}
(a): For prescribed cell masses, the best $W_2$ approximation uses cell means of the quantile function \citep[Theorem~5.5]{xu2019}. Indeed, these are the $L^2$ projection onto the space of functions constant on the prescribed intervals, and the resulting cell means remain nondecreasing. Hence the equal-mass residual error $\varepsilon_M^{\mathrm{eq}}$ is the optimal error among measures with $M$ equal weights. Absolute continuity of the quantile and the stated integrability condition allow \citet[Theorem~5.15]{xu2019} to be applied with $r=2$.

(b) Under the moment condition, Zador's theorem in one dimension
\citep[Thm.~6.2]{grafLuschgy2000} gives
$\lim_{M\to\infty}M^2(\varepsilon_M^\star)^2=\frac1{12}\bigl(\int f^{1/3}\bigr)^3$.
To compare this constant with the one in (a), apply H\"older's inequality with
exponents $3$ and $3/2$ to $f^{1/3}=f^{-1/3}f^{2/3}$ on $\{f>0\}$:
\[
 \int f^{1/3}\le\Bigl(\int_{\{f>0\}}f^{-1}\Bigr)^{1/3}\Bigl(\int f\Bigr)^{2/3}
 =\Bigl(\int_{\{f>0\}}f^{-1}\Bigr)^{1/3}.
\]
Cubing gives the stated inequality; it is trivial when $\int_{\{f>0\}}f^{-1}=\infty$.
Equality in H\"older's inequality holds if and only if $f^{-1}$ and $f$ are
proportional on $\{f>0\}$, that is, $f$ is constant there. An absolutely continuous
quantile function has no jumps, so $\{f>0\}$ is an interval up to a null set, and
equality holds exactly when $\rho$ is uniform on an interval. This comparison concerns
the limiting constants, not the errors at finite $M$.

(c) By symmetry and scaling it suffices to treat the upper tail of $N(0,1)$. Number
the equal-mass cells from the right, so that cell $k$ has tail probabilities
$((k-1)/M,k/M]$. The outermost cell is the tail beyond $x_1=\Phi^{-1}(1-1/M)$, with
$x_1\sim\sqrt{2\log M}$, and the conditional variance of a Gaussian tail beyond $x$ is
$\asymp x^{-2}$ for large $x$. This cell alone contributes
$\frac1M\cdot x_1^{-2}\asymp(M\log M)^{-1}$, which gives the lower bound. 
For $2\le k\le M/4$, cell $k$ lies to the right of $x_k=\Phi^{-1}(1-k/M)\ge
\Phi^{-1}(3/4)>0$, and $x_k^2\asymp\log(M/k)$ for $k\le M/e$. By the Mills-ratio
bounds  $\frac{x}{1+x^2}\varphi(x)\le1-\Phi(x)\le\frac{\varphi(x)}{x}$ for $x>0$ \citep{gordon1941values}, its width is $\asymp(kx_k)^{-1}$, and its conditional
variance is at most the square of its width. The cells with $M/4<k\le 3M/4$ lie in
$[-\Phi^{-1}(3/4),\Phi^{-1}(3/4)]$, where the density is at least
$\varphi(\Phi^{-1}(3/4))$, so each has width $O(M^{-1})$ and together they contribute
$O(M^{-2})$.  Hence
\[
 (\varepsilon_M^{\mathrm{eq}})^2\lesssim\frac1M\Bigl(\frac1{\log M}
 +\sum_{2\le k\le M/2}\frac{1}{k^2\max\{\log(M/k),1\}}\Bigr)+O(M^{-2})
 \asymp\frac1{M\log M},
\]
since the sum is dominated by $k\le\sqrt M$, where $\log(M/k)\ge\tfrac12\log M$. %

(d): For the uniform ball, symmetry reduces the calculation to $t\in[0,1]$. There
$f_\theta(t)=c_d(1-t^2)^{(d-1)/2}\asymp(1-t)^{(d-1)/2}$, with constants depending on $d$. Thus the upper-tail mass satisfies $\bar F_\theta(t)\asymp(1-t)^{(d+1)/2}$. Put $\alpha=2/(d+1)$ and number the equal-mass cells from the right. For $k\le\lfloor M/2\rfloor$, cell $k$ covers tail probabilities $((k-1)/M,k/M]$ and has width
\[
 w_k\asymp M^{-\alpha}k^{\alpha-1}.
\]
For $k=1$ this follows directly from the tail-mass relation. For $k\ge2$, it follows by integrating the derivative of the inverse tail function, whose magnitude is comparable to $u^{\alpha-1}$. The conditional variance on each cell is comparable to $w_k^2$: its upper bound is $w_k^2/4$, and a lower bound with a positive $d$-dependent constant follows by comparison with the density $(1-t)^{(d-1)/2}$ and rescaling the cell. For this power density the normalized variance is bounded away from zero, including at the endpoint.

Each cell has mass $1/M$, so symmetry yields
\[
 (\varepsilon_M^{\mathrm{eq}})^2\asymp
 M^{-1-\frac4{d+1}}\sum_{k=1}^{\lfloor M/2\rfloor}
 k^{\frac4{d+1}-2}+O(M^{-2}),
\]
where any remaining central cell has bounded density and contributes at most $O(M^{-2})$. The sum is bounded for $d>3$, grows like $\log M$ for $d=3$, and grows like $M^{4/(d+1)-1}$ for $d<3$. Substitution gives the three stated squared-error rates. Zador's formula gives $(\varepsilon_M^\star)^2\asymp M^{-2}$ with a finite positive constant for every fixed $d$. \end{proof}

\subsection{Proof of \cref{prop:moments}}
Fix $\theta$ and suppress it from the notation for simplicity. Let $U$ be uniform on $(0,1)$ and set $X_\theta=Q_{\theta_\#\mu}(U)$ and $\hat X_\theta=\hat Q_{\theta_\#\mu}(U)$. Thus $X_\theta$ has the law of $\theta^\top X$, $\hat X_\theta$ has the sketch law, and this coupling realizes the one-dimensional quadratic Wasserstein distance. Let $C=j$ when $U\in I_j$, write $q_j=\abs{I_j}>0$, and let $\mathcal G=\sigma(C)$. By the centroid condition,
\begin{equation}\label{eq:proof-centroid-stationarity}
 \hat X_\theta=b_C=\E[X_\theta\mid\mathcal G].
\end{equation}
For $Z=X_\theta-\hat X_\theta$, it follows that $\E[Z\mid\mathcal G]=0$. Conditional Jensen gives $\E\hat X_\theta^2\le\E X_\theta^2<\infty$; hence $\hat X_\theta,Z\in L^2$. Consequently,
\[
 \E\hat X_\theta=\E X_\theta,\qquad
 \E[\hat X_\theta Z]=0,\qquad
 \E X_\theta^2=\E\hat X_\theta^2+\E Z^2.
\]
These prove \eqref{eq:centroid-moment-identities}, since $\E Z^2=\varepsilon_{\mu,\theta}^2$ \citep[Prop.~1.7, Cor.~1.8]{lemaire2020}.

For $f\in C^2(\R)$ with bounded second derivative, Taylor's formula with integral remainder gives
\[
 f(X_\theta)-f(\hat X_\theta)=f'(\hat X_\theta)Z+Z^2\int_0^1(1-t)f''(\hat X_\theta+tZ)\,dt.
\]
Bounded $f''$ implies that $f'$ has at most linear growth and $f$ at most quadratic growth. Thus all terms are integrable, and
\[
 \E[f'(\hat X_\theta)Z]=\E\bigl[f'(\hat X_\theta)\E[Z\mid\mathcal G]\bigr]=0.
\]
Taking expectations and using $\int_0^1(1-t)\,dt=1/2$ proves $\abs{\E f(X_\theta)-\E f(\hat X_\theta)}\le\tfrac12\norm{f''}_\infty\E Z^2$ \citep{pages1998,pages2003,lemaire2020}. This argument uses only the centroid condition, not global optimality.

Now assume $\norm X\le R$ almost surely. Since $\theta$ is a unit vector, $\abs{X_\theta}\le R$ almost surely, and \eqref{eq:proof-centroid-stationarity} implies $\abs{\hat X_\theta}\le R$ almost surely. The segment between $\hat X_\theta$ and $X_\theta$ therefore lies in $[-R,R]$. Applying the same Taylor argument to $f(x)=x^r$ on this interval, where $\sup_{\abs x\le R}\abs{f''(x)}=r(r-1)R^{r-2}$, proves \eqref{eq:uncorrected-moment-bound}. For even $r$, $x\mapsto x^r$ is convex on $\R$, so conditional Jensen also gives
\[
 \hat X_\theta^r=\bigl(\E[X_\theta\mid\mathcal G]\bigr)^r\le\E[X_\theta^r\mid\mathcal G].
\]
Taking expectations proves the asserted one-sided inequality.

To establish the correction, observe that
\begin{equation}\label{eq:proof-cell-energies}
 e_j=\E[Z^2\mathbf 1_{\{C=j\}}],\qquad
 \sum_j e_j=\E Z^2=\varepsilon_{\mu,\theta}^2.
\end{equation}
In particular, these are unnormalized per-cell energies; the conditional variance in cell $j$ is $e_j/q_j$. Since $\hat X_\theta=b_j$ on $\{C=j\}$,
\[
 \E[\hat X_\theta^{r-2}Z^2]=\sum_j b_j^{r-2}e_j.
\]
For $r=2$, formula \eqref{eq:corrected-moment-estimate} therefore coincides with \eqref{eq:centroid-moment-identities} and is exact. For $r\ge3$, Taylor's formula through the quadratic term gives
\begin{align*}
 X_\theta^r
 &=\hat X_\theta^r+r\hat X_\theta^{r-1}Z+\binom r2\hat X_\theta^{r-2}Z^2+\mathcal R_r,\\
 \mathcal R_r
 &=\frac{Z^3}{2}\int_0^1(1-t)^2\frac{d^3}{dx^3}x^r\bigg|_{x=\hat X_\theta+tZ}\,dt.
\end{align*}
As $\abs{\hat X_\theta+tZ}\le R$ and $\int_0^1(1-t)^2\,dt=1/3$,
\[
 \abs{\mathcal R_r}\le\binom r3R^{r-3}\abs Z^3.
\]
The linear term again has zero expectation. The expectation of the first and quadratic terms is exactly $\tilde m_{r,\theta}$, and hence
\begin{equation}\label{eq:proof-third-residual-bound}
 \abs{m_{r,\theta}-\tilde m_{r,\theta}}\le\binom r3R^{r-3}\E\abs Z^3.
\end{equation}
If $\abs Z\le h_\theta$ almost surely, then $\E\abs Z^3\le h_\theta\E Z^2$, which proves \eqref{eq:corrected-moment-bound}. Such a bound always exists under bounded support, since $\abs Z\le2R$; the sharper rate below requires a bound that decreases with $M$.

For completeness, the exact moment decomposition underlying this correction is
\[
 m_{r,\theta}=\hat m_{r,\theta}+\sum_{k=2}^r\binom rk\sum_j b_j^{r-k}e_{j,k},\qquad
 e_{j,k}:=\E[Z^k\mathbf 1_{\{C=j\}}].
\]
It follows by binomial expansion and $e_{j,1}=0$. The correction retains only the $k=2$ term, and \eqref{eq:proof-third-residual-bound} bounds the omitted terms jointly. It therefore requires no evaluation of residual moments of order three or higher.

Next let $\operatorname{Sym}^r(\R^d)$ carry the Frobenius inner product, and let $\R^L$ carry the Euclidean inner product. For $T_r=\E X^{\otimes r}$,
\[
 \langle T_r,\theta_\ell^{\otimes r}\rangle=\E(\theta_\ell^\top X)^r=m_{r,\theta_\ell}.
\]
Thus the true moment vector $m_r=(m_{r,\theta_\ell})_{\ell=1}^L$ satisfies $m_r=A_rT_r$. Define
\[
 \hat T_r=A_r^\dagger(\hat m_{r,\theta_\ell})_{\ell=1}^L,\qquad
 \tilde T_r=A_r^\dagger(\tilde m_{r,\theta_\ell})_{\ell=1}^L.
\]
If $A_r$ has full column rank, then $A_r^\dagger A_r$ is the identity on $\operatorname{Sym}^r(\R^d)$. Therefore, for either input moment vector $v$ and its least-squares reconstruction $T(v)$,
\[
 T(v)-T_r=A_r^\dagger(v-m_r),\qquad
 \norm{T(v)-T_r}_F\le\norm{A_r^\dagger}_{2\to F}\norm{v-m_r}_2.
\]
Applying the coordinatewise bounds \eqref{eq:uncorrected-moment-bound} and \eqref{eq:corrected-moment-bound} proves \eqref{eq:uncorrected-tensor-bound} and \eqref{eq:corrected-tensor-bound}, respectively. Here $\varepsilon_{\mu,\ell}=\varepsilon_{\mu,\theta_\ell}$.

To verify the generic rank claim, put $n_r=\dim\operatorname{Sym}^r(\R^d)=\binom{d+r-1}{r}$. The tensors $\theta^{\otimes r}$, $\theta\in\Sph$, span this space. Indeed, a symmetric tensor $T$ orthogonal to all of them defines the homogeneous polynomial $p_T(x)=\langle T,x^{\otimes r}\rangle$, which vanishes on the sphere. Homogeneity implies that $p_T$ vanishes on all of $\R^d$, and uniqueness of polynomial coefficients gives $T=0$. Consequently, some $n_r$ directions yield an invertible evaluation matrix. In a fixed basis, its determinant is a polynomial in the direction coordinates and does not vanish identically on $(\Sph)^{n_r}$. For $d\ge2$, its zero set is a proper algebraic subset of this product and has product surface measure zero; this is the exceptional set excluded by ``generic.'' For $d=1$, $n_r=1$ and every unit direction already gives a nonzero evaluation. Adding rows preserves full column rank. This is the polynomial reconstruction argument of \citet{bodmann2018}.

Finally, exact mean and covariance recovery require only $\mu\in\Pcal_2(\R^d)$. The directional means are $\theta_\ell^\top\E X$, so $d$ spanning directions determine $\E X$ uniquely. From the stored total residual energies and \eqref{eq:centroid-moment-identities}, the directional second moments are known exactly. If $A_2$ is injective, these determine $T_2=\E XX^\top$ exactly. Injectivity of $A_2$ also implies that the directions span $\R^d$: otherwise, for a nonzero $v$ orthogonal to every direction,
\[
 A_2(vv^\top)=\bigl((\theta_\ell^\top v)^2\bigr)_{\ell=1}^L=0,
\]
contradicting injectivity. The mean is therefore also recovered, and $\operatorname{Cov}(X)=T_2-(\E X)(\E X)^\top$ is exact. The generic rank result with $r=2$ requires $L\ge\binom{d+1}{2}$.

For the stated rates, fix $r$, $R$, and the directions. The bound $\varepsilon_{\mu,\theta}=O(M^{-1})$ in \eqref{eq:uncorrected-moment-bound} gives $O(M^{-2})$ error, while $h_\theta=O(M^{-1})$ together with \eqref{eq:corrected-moment-bound} gives $O(M^{-3})$ error. Under the bounded-density assumptions of \cref{lem:cellwidth}, every optimal Voronoi cell intersected with the support has width $O(M^{-1})$. Each centroid lies within its cell, so $\abs{X_\theta-\hat X_\theta}=O(M^{-1})$ almost surely. A uniform grid on the bounded support gives squared distortion $O(M^{-2})$; optimality gives the same upper bound for the optimal quantizer. Thus both conditions hold in that setting. For fixed directions with $A_r$ injective, the corresponding tensor bounds inherit the same rates. All identities refer to the law being quantized; if that law is empirical, they recover or approximate its empirical moments. \qed

\section{Implementation details}\label{app:impl}
This section gives the implementation details needed to reproduce the experiments of \cref{sec:experiments}. Benchmark experiments ran on a server with $128$ logical cores and four NVIDIA RTX A5000 GPUs.

\noindent\textbf{Main components.}
The optimal sketch of a slice is computed by exact dynamic programming. Weights are normalized to sum to one, zero-weight points are removed, and duplicate locations are merged after sorting; if at most $M$ distinct locations remain, the slice is returned unchanged. Sketch distances are computed by a monotone sweep that transports the smaller of the two current atom masses and advances the exhausted atom, in $O(M)$ time per slice. Uncompressed sliced distances use the same sweep on the sorted projections, or the mean squared difference of sorted projections when all distributions have the same number of equally weighted points.

\noindent\textbf{Classification (\cref{tab:compression}).}
Documents \citep{kusner2015} are bags of $300$-dimensional word embeddings, normalized to unit length and weighted by normalized word counts. We use the five provided train/test splits (one for Reuters) and hold out $20\%$ of each training split for validation. AML subjects \citep{aghaeepour2013} are represented by $10^4$ randomly subsampled cells of tube $6$ with all $7$ channels, each standardized with the mean and standard deviation of the training subjects of the fold. Validation scores come from a stratified $5$-fold cross-validation inside each training fold. For ModelNet40 shapes \citep{mn40} we use the official train/test split and hold out $20\%$ of the training shapes for validation. Table entries are means and standard deviations over splits or folds. A star marks a paired difference between ours and equal-mass: a paired $t$-test over the five document splits ($5\%$ level), a mean difference larger than two standard errors over the $25$ AML folds, and a $95\%$ paired bootstrap interval over the test shapes of ModelNet40 ($1000$ resamples) that excludes zero. The error ratio is the median over distributions and slices of the ratio of the equal-mass to the optimal residual energy, averaged over splits or folds.

The $L=16$ slices are linear maps with rows normalized to unit length, trained on the training split only (without the validation part) with the neighbourhood-component loss \citep{goldberger2004}: a softmax over $-a\,\SW_{2,\Theta}^2$ to the other distributions of the batch, with a learned scale $a$. We use Adam with learning rate $10^{-2}$. For documents and ModelNet40, batches have $256$ distributions, distances during training use $64$ quantile levels, and training runs for at most $3000$ steps with early stopping on the validation $k$NN error. For AML, the slices are trained on the training part of the first inner fold, with early stopping on its validation part. Each step uses all of these subjects, represented by random subsets of $512$ cells, for at most $300$ steps. The five-fold validation scores used to choose $k$ and $M$ therefore include subjects seen during slice training, and test subjects are never used. The slices are then frozen and shared by all methods. Each distance is used in a $k$-nearest-neighbour classifier, with $k\in\{1,\dots,19\}$ chosen on validation.

For documents, WMD uses the precomputed distances of \citet{sato2022}. For AML and ModelNet40, exact $W_1$ with Euclidean cost is computed by a network-simplex solver on the first $10^3$ cells and the first $256$ points, respectively. All baselines use the same splits, validation and $k$NN classifier.

\noindent\textbf{Learning from few subjects (\cref{tab:amlmain_logreg,tab:amlmain_svm}).}
This experiment uses the AML subjects ($N=10^4$ cells) and the $25$ folds of \cref{tab:compression}, with $L=32$ random directions (seed $0$) shared by all subjects. A subject is represented by a feature vector: its $LN=320{,}000$ sorted projections (uncompressed); for every slice, the $M$ atoms and $M-1$ masses of the optimal sketch (ours) or the $2M-1$ cell means of equal-mass cells (equal-mass). Inside each training fold, we draw $2$, $4$, $8$ or $16$ AML subjects together with normal subjects in the ratio of the dataset, ten times per size; ``all'' uses the whole training fold. On each draw we fit a class-balanced, $\ell_2$-regularized logistic regression with penalty $C=1$ or linear SVM, with one regularization setting for all representations, and compute the AUC on the test subjects of the fold. Tables report the mean and standard error over the $25$ folds.

\noindent\textbf{Pooling layer (\cref{fig:pooling}).}
We use the official code and ModelNet40 protocol of \citet{naderializadeh2021} and change only the pooling layer. We report the test accuracy after the last epoch, computed as in their code (mean over test batches of $32$ shapes), as mean and standard deviation over three seeds. The PSWE layer compares the sorted projections of each slice with a learned reference of $1024$ values and reduces the differences with $1024$ learned weights to one number per slice. Our variants replace the sorted projections by the step quantile of a sketch with $M=4$ cells, i.e., each projected point is replaced by the mean of its cell, using equal-mass cells ($256$ points each) or the optimal partition. 
GAP, PMA \citep{lee2019settransformer} and FSPool \citep{zhang2020fspool} are taken from the same code and pool the backbone features directly.

\input{theory_validation_appendix}

\section{Additional experiments}\label{app:exp}

\subsection{Sweeps over $L$ and $M$}\label{app:sweeps}
\Cref{fig:sweepM-docs,fig:sweepM-rest,fig:sweepL-docs,fig:sweepL-rest} compare the optimal sketch with $M$ atoms and equal-mass cells with $M$ cells at the same $L$, as in \cref{tab:compression}, on the learned slices and runs behind \cref{tab:compression}, for every $L$ and $M$. Every configuration is shown, each with its own $k$ chosen on validation, and no further selection is made. The optimal sketch stores $L(2M-1)$ numbers per distribution and equal-mass cells $LM$. \Cref{fig:sweepM-docs,fig:sweepM-rest} fix $L$ and sweep $M$; \cref{fig:sweepL-docs,fig:sweepL-rest} fix $M$ and sweep $L$. Bands are standard errors over splits or folds (documents, AML) and bootstrap $95\%$ intervals over test shapes (ModelNet40); the dotted line is the uncompressed $\SW_{2,\Theta}$ on the same slices. Over the $23$ pairs of benchmark and $L$, the optimal sketch is more accurate in all $23$ at $M=4$ and in $21$ at $M=8$, by $0.8$ and $0.6$ points on average; six differences at each $M$ are distinguishable in its favour, and one at $M=8$ in favour of equal-mass cells. At $M=2$ it is ahead in $15$ pairs, by $2.0$ points on average and by $10$ points on ModelNet40, but behind on Recipe at every $L$, by $2.3$ points. From $M=16$ on, the mean difference is $0.1$ points or less, and both rules are close to the uncompressed distance.

\begin{figure}[!htbp]
\centering
\includegraphics[width=\textwidth]{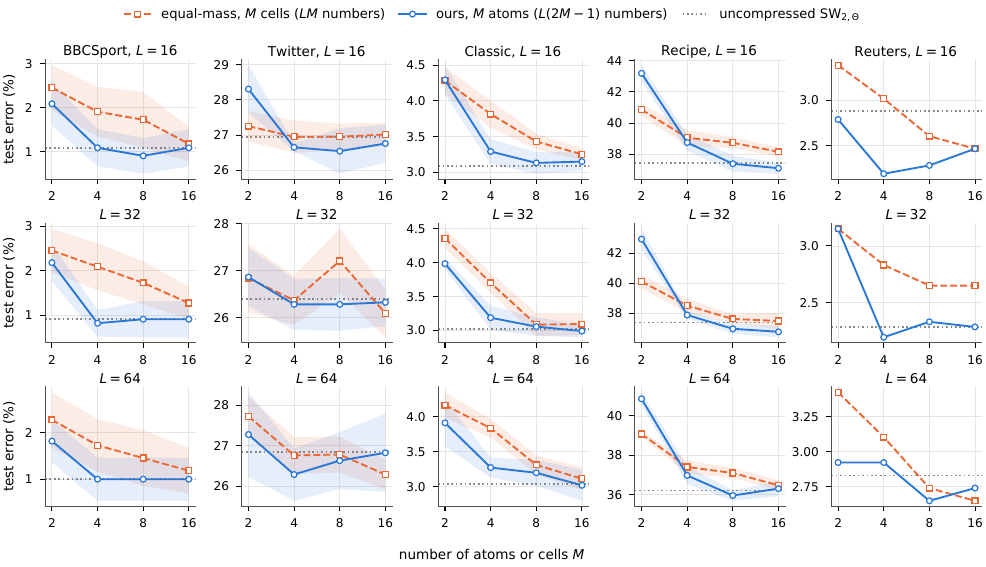}
\caption{Documents, learned slices: optimal sketch with $M$ atoms against equal-mass cells with $M$ cells, fixed $L$ (rows), sweep over $M$; test error.}
\label{fig:sweepM-docs}
\end{figure}

\begin{figure}[!htbp]
\centering
\includegraphics[width=\textwidth]{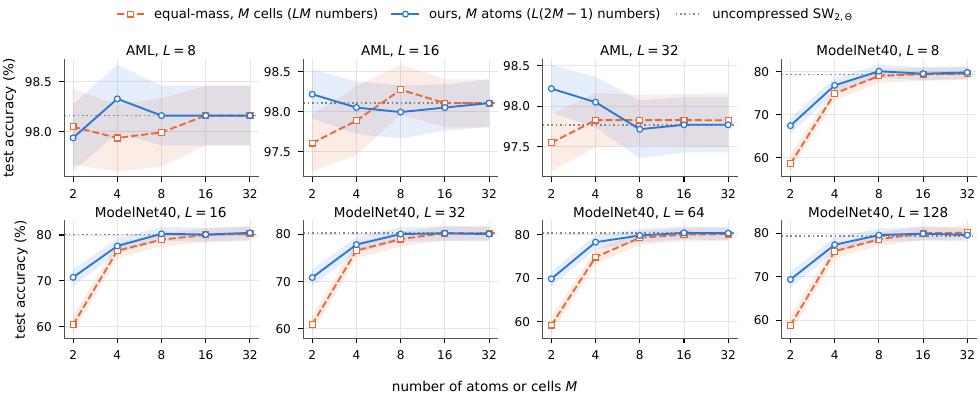}
\caption{AML and ModelNet40, learned slices: optimal sketch with $M$ atoms against equal-mass cells with $M$ cells, fixed $L$, sweep over $M$; test accuracy.}
\label{fig:sweepM-rest}
\end{figure}

\begin{figure}[!htbp]
\centering
\includegraphics[width=\textwidth]{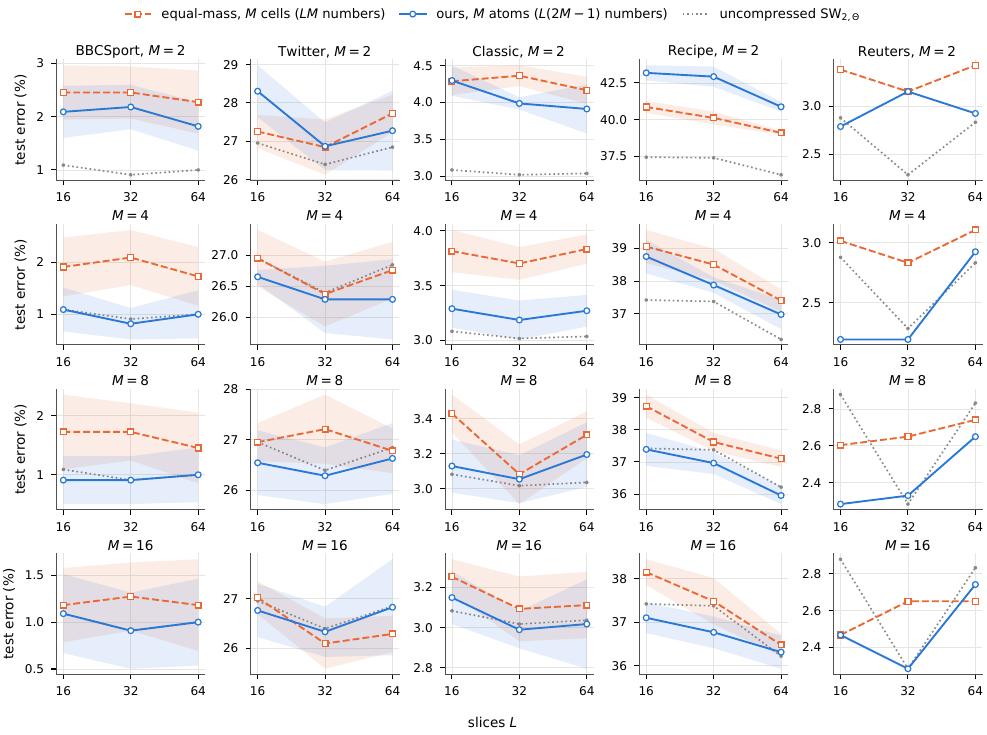}
\caption{Documents, learned slices: optimal sketch with $M$ atoms against equal-mass cells with $M$ cells, fixed $M$ (rows), sweep over $L$; test error.}
\label{fig:sweepL-docs}
\end{figure}

\begin{figure}[!htbp]
\centering
\includegraphics[width=\textwidth]{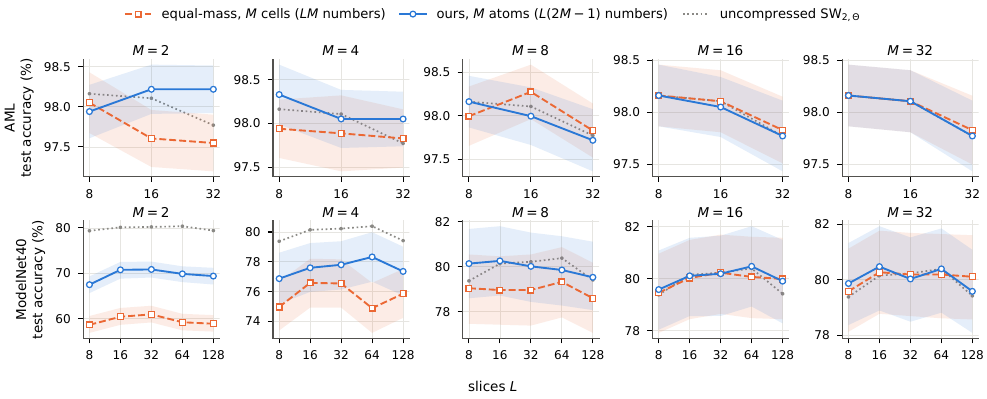}
\caption{AML and ModelNet40, learned slices: optimal sketch with $M$ atoms against equal-mass cells with $M$ cells, fixed $M$, sweep over $L$; test accuracy.}
\label{fig:sweepL-rest}
\end{figure}

\subsection{Few labelled subjects: linear SVM}\label{app:svm}
\Cref{tab:amlmain_svm} repeats the comparison of \cref{tab:amlmain_logreg} with a class-balanced linear SVM in place of logistic regression, on the same slices, folds and draws of training subjects. The conclusions are those of logistic regression. The optimal sketch has a higher test AUC than the uncompressed representation at every training size, and the paired difference falls from $3.9\pm0.5$ points with two AML subjects to $0.9\pm0.3$ with the whole fold. 

\begin{table}[!htbp]
\centering\small
\caption{\textbf{AML, linear SVM from few labelled subjects.} As \cref{tab:amlmain_logreg}, with a class-balanced linear SVM in place of logistic regression. Stored values: number of values stored per subject.}
\label{tab:amlmain_svm}
\begin{tabular}{@{}lcccccc@{}}
\toprule
 & & \multicolumn{5}{c}{AML subjects in training}\\
\cmidrule(l){3-7}
summary & stored values & 2 & 4 & 8 & 16 & all\\
\midrule
uncompressed & 320{,}000 & $92.2\pm0.6$ & $94.9\pm0.4$ & $96.8\pm0.4$ & $97.9\pm0.4$ & $98.5\pm0.4$\\
mean, covariance & 35 & $92.6\pm0.7$ & $95.4\pm0.6$ & $96.1\pm0.7$ & $97.0\pm0.6$ & $97.2\pm0.5$\\
\midrule
equal-mass & 96 & $91.2\pm0.7$ & $93.7\pm0.4$ & $95.7\pm0.5$ & $96.7\pm0.5$ & $97.8\pm0.5$\\
ours & 96 & $\mathbf{96.1\pm0.4}$ & $\mathbf{97.4\pm0.3}$ & $\mathbf{98.4\pm0.2}$ & $\mathbf{99.0\pm0.2}$ & $\mathbf{99.4\pm0.2}$\\
\addlinespace
equal-mass & 480 & $92.0\pm0.6$ & $94.7\pm0.4$ & $96.7\pm0.4$ & $97.8\pm0.4$ & $98.4\pm0.4$\\
ours & 480 & $\mathbf{96.2\pm0.4}$ & $\mathbf{97.5\pm0.3}$ & $\mathbf{98.2\pm0.3}$ & $\mathbf{98.8\pm0.3}$ & $\mathbf{98.9\pm0.3}$\\
\bottomrule
\end{tabular}
\end{table}

\end{document}

%% file: theory_validation_appendix.tex
\section{Numerical validation of the theoretical results}\label{app:theory-validation}
\input{results/theory_validation/numbers.tex}

We give a separate numerical check of every theorem, proposition, corollary,
lemma and example in the paper. These experiments check exact identities and
finite inequalities, and illustrate the asymptotic statements on distributions
whose assumptions are known analytically. Finite computations do not establish
a universal bound or an asymptotic theorem. In particular, an $O(M^{-p})$ upper
bound need not be attained. \Cref{tab:theory-coverage} maps the results to the
experiments and their recorded data.

\begin{table}[!htbp]
\centering\small\setlength{\tabcolsep}{4pt}
\caption{Theory-to-experiment correspondence. CSV files are in
\texttt{results/theory\_validation/}; all are regenerated by one command.}
\label{tab:theory-coverage}
\begin{tabular}{@{}p{.23\textwidth}p{.34\textwidth}p{.38\textwidth}@{}}
\toprule
Result & Numerical check & Recorded CSV files\\
\midrule
\Cref{thm:identity} & Independent residual/boundary decomposition; both signs of the distance gap & \texttt{certificates\_atomic}\\
\Cref{cor:certificates} & $H_{\rm cell}\le H_0$, signed bounds and complete distance interval & \texttt{certificates\_atomic}\\
\Cref{thm:sliced} & Averaged signed and unsquared bounds; shared linear and nonlinear maps & \texttt{certificates\_sliced}\\
\Cref{thm:rate} & Pierce scaling, Gaussian sphere supremum, Zador constants & \texttt{continuous\_rates}, \texttt{sliced\_rate}\\
\Cref{prop:equalmass} & Gaussian and all three ball-projection rate regimes & \texttt{continuous\_rates}\\
\Cref{cor:lipschitz} & Known quantile Lipschitz constants and nonmatching cells & \texttt{lipschitz}\\
\Cref{lem:cellwidth} & Optimal-cell width, mass, pointwise error and residual bounds & \texttt{cell\_widths}\\
\Cref{cor:density} & Squared-distance gap and the stated constant $K_{\mu\nu}$ & \texttt{bounded\_density}\\
\Cref{ex:sharp} & $M^{3/2}(W_2^2-D^2)\to-1$; expansion exceeds $S^2$ & \texttt{sharpness}\\
\Cref{prop:moments} & Scalar/tensor bounds, correction, exact recovery, rank and conditioning & \texttt{moments\_scalar}, \texttt{moments\_tensor}, \texttt{moments\_rates}, \texttt{moments\_gaussian\_recovery}, \texttt{moments\_rank}, \texttt{moments\_conditioning}\\
Sampling/directions & Empirical-population decomposition and $L^{-1/2}$ Monte Carlo errors & \texttt{certificates\_sampling}, \texttt{certificates\_directions}\\
\bottomrule
\end{tabular}
\end{table}

\subsection{Partition mismatch and computable certificates}
\paragraph{Identity (\cref{thm:identity}).}
We generate weighted Gaussian, separated-mixture, Student-$t_3$, atomic and
degenerate pairs, with $37$--$60$ and $48$--$71$ input points. The inputs include
duplicate locations and zero weights. For each of $24$ trials per family, we
use optimal sketches, equal-mass sketches, a shared nonuniform mass partition,
and two independently chosen nonuniform partitions. The first cell count cycles
through $1,2,4,8$; the optimal and independent-partition pairs use one additional
cell for the second measure. Nonuniform masses are Dirichlet with parameter
$2$ per cell. Optimal sketches may use fewer atoms when the input support is small.

All quantile integrals are finite sums on the common refinement of the two input
and two sketch partitions. The distance gap, residual norm, boundary terms and
per-cell energies are evaluated separately, and the resulting energies and
transport costs are checked against the production routines. Across
\TVIdentityChecks{} pairs, the largest identity discrepancy divided by
$1+W_2^2+D^2+S^2+H_0$ is \TVIdentityError. There are \TVExpansions{}
expansions and \TVContractions{} contractions exceeding absolute tolerance
$2\cdot10^{-10}$, with the remaining cases numerically exact. Residual primitives
are nonpositive and vanish at their own cell boundaries; shared partitions
have zero boundary term and contract the distance. As construction checks,
$24$ small problems agree with exhaustive contiguous-partition enumeration,
and an atomic input is recovered exactly once the budget reaches its support size.

\paragraph{Certificates (\cref{cor:certificates}).}
On the same pairs we independently evaluate $H_0$, $H_{\rm cell}$, both signed
inequalities and both endpoints of \eqref{eq:distance-certificate-interval}.
No inequality is violated beyond the scaled tolerance $2\cdot10^{-10}$;
the largest positive scaled discrepancy is \TVMaxViolation.
\Cref{fig:theory-certificates} displays the identity and the location of each
nonexact optimal-sketch gap within its certified interval. Pairs with $S^2+H_{\rm cell}$ below the tolerance, which
are recovered exactly, remain checked but are omitted from both normalized panels.

\begin{figure}[!htbp]
\centering
\includegraphics[width=\textwidth]{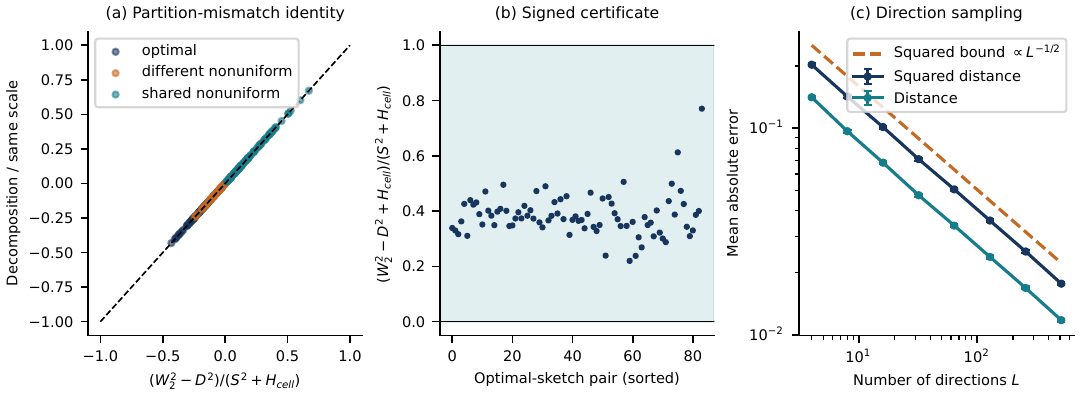}
\caption{Exact atomic identity and certificate checks, and direction sampling.
Left: independently computed sides of the identity coincide. Middle: the signed
gap lies in $[-H_{\rm cell},S^2]$, rescaled to $[0,1]$. Right: Monte Carlo mean
absolute errors with bars of two standard errors over \TVDirectionReplicates{}
independent direction sets; the dashed line is the analytic squared-distance bound.}
\label{fig:theory-certificates}
\end{figure}

\subsection{Sliced, sampling and direction certificates}
\paragraph{Sliced bounds (\cref{thm:sliced}).}
We use $16$ weighted pairs in $\R^3$, with $67$ and $83$ points, and eight shared
slices. Three choices give \TVSlicedChecks{} comparisons: random linear
directions, linear directions fitted from the mean difference of separate
synthetic training samples and then frozen, and shared nonlinear maps
$g_\ell(x)=\theta_\ell^\top x+0.15(\theta_\ell^\top x)^2$.
Optimal four-atom sketches are paired with equal-mass five-cell sketches.
Both averaged signed bounds and the unsquared bound pass in every case.
The data-dependent directions test the theorem's applicability after selecting
shared maps; this is not a test of the quality of a learning algorithm.

\paragraph{Sampling and direction errors.}
For the deterministic sampling decomposition, two finite weighted populations
of $9$ and $11$ points in $\R^3$ make the population reference exactly computable.
We draw $n=16,64,256,1024$ observations, with $16$ repetitions at each size,
and build four-atom sketches on eight directions. All \TVSamplingChecks{}
errors are below $E_\mu^n+E_\nu^n+S_\mu^n+S_\nu^n$; the largest ratio is
\TVSamplingRatio. This checks the decomposition without imposing a sampling-rate claim.

For direction sampling alone, let $\mu=N(0,I_3)$ and $\nu=N(a,I_3)$ with
$\norm a=1.3$. Then $Z(\theta)=(\theta^\top a)^2$,
$\E Z=\norm a^2/3$ and $\Var Z=\norm a^4(1/5-1/9)$ are known exactly.
For $L=4,8,\ldots,512$, we use \TVDirectionReplicates{} independent sets of
directions. Log--log slopes over $L=64,128,256,512$ are
\TVDirectionSquaredSlope{} for the squared-distance mean absolute error and
\TVDirectionSlope{} for the distance error, consistent with $L^{-1/2}$.
The script checks the Monte Carlo estimates against the analytic expectation
bounds with a five-standard-error allowance; the plot reports two-standard-error bars.

\subsection{Optimal and equal-mass quantization rates}
\paragraph{Continuous numerical protocol.}
We use $N(0,1)$ and the one-dimensional projections of uniform balls in
$d=1,2,3,5,9$. Thus there is no empirical quantile grid or sampling floor.
Cell masses and first moments are analytic. Optimal quantizers solve the
continuous centroid/Voronoi equations for $M=4,8,\ldots,512$; global optimality
of an exact solution follows from log-concavity for these particular laws
\citep{trushkin1982,kieffer1983}. We report finite-precision solutions, not
interval-certified optima. The largest centroid residual is \TVRootResidual.
Centered cell integration uses $32$ and $64$ Gauss nodes, analytic Gaussian
tail variances and a change of variables at ball endpoints. The largest
relative energy change is \TVQuadratureChange. A perturbed initialization,
tighter root tolerance and $96$ nodes at $M=512$ change optimal energies by
at most \TVSolverChange. These checks are recorded in
\texttt{continuous\_refinement.csv}.

\paragraph{Free weights (\cref{thm:rate}).}
We plot $M^2(\varepsilon_M^\star)^2/C_{\rm Zador}$, where
$C_{\rm Zador}=\frac1{12}(\int f^{1/3})^3$ is evaluated analytically.
All six curves approach one (\cref{fig:theory-rates,tab:theory-rates}).
For the standard Gaussian $C_{\rm Zador}=\sqrt3\pi/2$.
We also compare $N(0,\operatorname{diag}(1,4))$ and
$N((0.4,-0.3)^\top,\operatorname{diag}(2,0.7))$ on seven equally spaced
directions in the plane. Direct transport between the scaled Gaussian centroid
sketches is compared with the analytic sliced distance. Its error is below
$E_\mu+E_\nu$, and $M(E_\mu+E_\nu)$ stabilizes; the actual error in this
example decays faster. For the first Gaussian the sphere supremum is exactly
$2\varepsilon_M^\star(N(0,1))$ and $\E\norm X^4=59$, so the recorded ratio
$M\sup_\theta\varepsilon_M^\star(\theta_\#\mu)/59^{1/4}$ also directly
checks the uniform-direction scaling. We do not estimate the universal
Pierce constant from these examples.

\paragraph{Equal masses (\cref{prop:equalmass}).}
Using the same continuous laws, with exact cell boundaries $Q(j/M)$, we extend
the equal-mass experiment to $M=2^{18}$. The normalization of the squared error
is $M\log M$ for the Gaussian, $M^2$ for balls in $d=1,2$,
$M^2/\log M$ for $d=3$, and $M^{(d+5)/(d+1)}$ for $d=5,9$.
The rescaled curves remain of constant order, while the equal/free error ratio
grows for the Gaussian and $d\ge3$. These are finite-range illustrations;
logarithmic corrections converge slowly. The uniform law gives identical
free and equal-mass errors and the exact rescaled value $1/3$.
For $d=2$, the finite equal-mass limit is $\pi^2/24\simeq0.41123$;
the observed $M^2(\varepsilon_M^{\rm eq})^2$ at $M=2^{18}$ is
\TVEqualBallTwoError, approaching this constant from below.

\begin{figure}[!htbp]
\centering
\includegraphics[width=.94\textwidth]{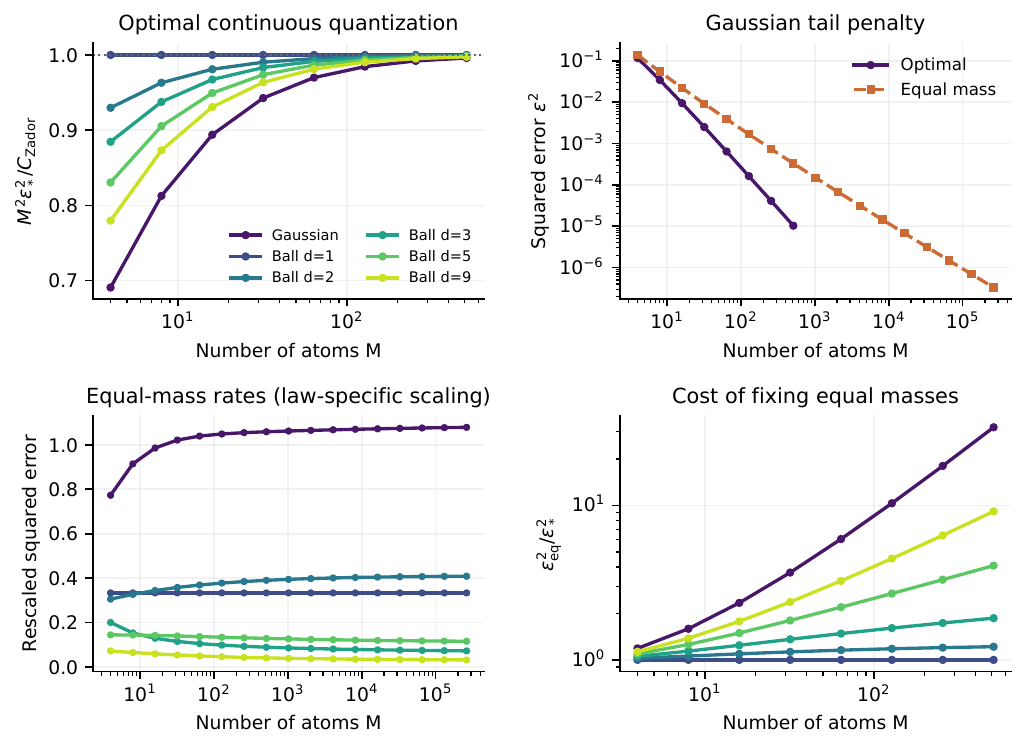}
\caption{Continuous quantization rates. Top left: comparison with analytic
Zador constants. Top right: Gaussian squared errors. Bottom left: equal-mass
errors under the law-specific normalizations in \cref{tab:theory-rates}; colors
match the top-left legend. Bottom right: the cost of prescribing equal masses
at the same $M$, on the common range $M\le512$.}
\label{fig:theory-rates}
\end{figure}

\begin{table}[!htbp]
\centering\footnotesize\setlength{\tabcolsep}{4pt}
\caption{Endpoints of the continuous rate experiment: $M=512$ for optimal
quantizers and $M=2^{18}$ for equal-mass cells. The last column is the
equal-mass squared error multiplied by the stated normalization; it is not
claimed to equal its limiting constant.}
\label{tab:theory-rates}
\input{results/theory_validation/rates_table.tex}
\end{table}

\subsection{Regularity and optimal-cell geometry}
\paragraph{Lipschitz quantiles (\cref{cor:lipschitz}).}
Take $Q_\mu(u)=u+0.2u^2$ and $Q_\nu(u)=0.3+0.7u+0.4u^3$, with known
constants $A_\mu=1.4$, $A_\nu=1.9$. Use equal-mass cells for $\mu$ and
alternating masses proportional to $1+0.35(-1)^{j-1}$, $j=1,\ldots,M$, for $\nu$.
Polynomial quadrature evaluates the energies and true distance exactly up
to rounding. For $M=6,8,16,\ldots,512$, both energy bounds and both signed
distance bounds hold. \Cref{fig:theory-regularity} shows the predicted
$M^{-2}$ squared-distance scale with genuinely mismatched partitions.

\paragraph{Cell widths (\cref{lem:cellwidth}).}
For $f_+(x)=e^x/(e-1)$ and $f_-(x)=e^{1-x}/(e-1)$ on $[0,1]$,
$c=1/(e-1)$ and $C=e/(e-1)$ are analytic density bounds. For each optimum at
$M=6,8,16,\ldots,512$, we check the largest clipped Voronoi width, cell mass,
pointwise residual and total residual against the lemma's stated constants.
All checks pass. For $f_+$ at $M=512$, $Mw_{\max}=\TVWidthScaled$ and
$Mq_{\max}=\TVMassScaled$, illustrating the asserted $M^{-1}$ scales.
The quantile Lipschitz bound follows directly from the known density minimum;
it is not estimated from samples.

\paragraph{Bounded densities (\cref{cor:density}).}
We compare the two preceding optimal quantizers and integrate the true
quantile distance independently. The explicit constant is
$K_{\mu\nu}=\TVDensityConstant$, and every squared gap is below
$K_{\mu\nu}/M^2$. At $M=512$, $M^2|W_2^2-D^2|=\TVDensityScaledGap$.
The general constant is conservative: the plot additionally gives the much
smaller computable $M^2H_A$, while the CSV retains the explicit $K_{\mu\nu}$.

\begin{figure}[!htbp]
\centering
\includegraphics[width=\textwidth]{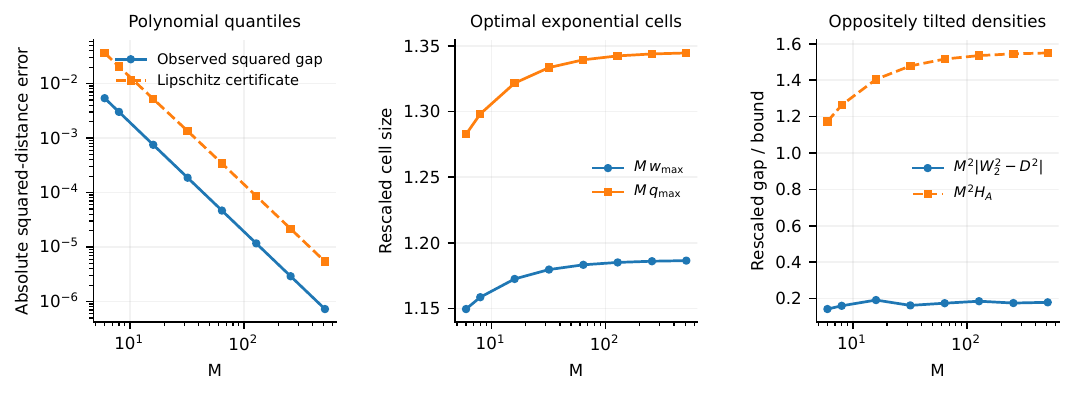}
\caption{Checks under analytic regularity assumptions: polynomial quantiles
and their Lipschitz certificate; optimal widths and masses for $f_+$; and the
squared-distance gap between $f_+$ and $f_-$ with the Lipschitz boundary budget.
No empirical quantile is treated as Lipschitz.}
\label{fig:theory-regularity}
\end{figure}

\subsection{Sharpness of the boundary budget}
For \cref{ex:sharp}, we use exactly the stated $M$-dependent pair, with $M$
even, $q=1/(M+1)$, $h=M^{-1/2}$, and equal-mass sketches with $M$ and $M+1$
cells. The two quantiles are piecewise affine or constant. Splitting at every
atom and cell boundary makes two-point Gaussian quadrature exact for the
squared quantities; four-point integration independently checks the result.
The energies also agree with the closed forms
\[
 \varepsilon_\mu^2=\frac1{3M^2},\qquad
 \varepsilon_{\nu_M}^2=\frac{1-q}{3M^2}+\frac{qh^2}{4}.
\]
For $M=8,16,\ldots,2^{16}$, the rescaled signed gap approaches $-1$.
At $M=\TVSharpM$, it equals \TVSharpLimit{}, while
$(D^2-W_2^2)/S^2=\TVSharpRatio$ (\cref{fig:theory-sharpness}).
The remaining terms after the central boundary contribution have bounded
$M^2$ scaling, as recorded in the CSV. This experiment concerns general
centroid sketches; it does not establish a counterexample for optimal quantizers.

\begin{figure}[!htbp]
\centering
\includegraphics[width=\textwidth]{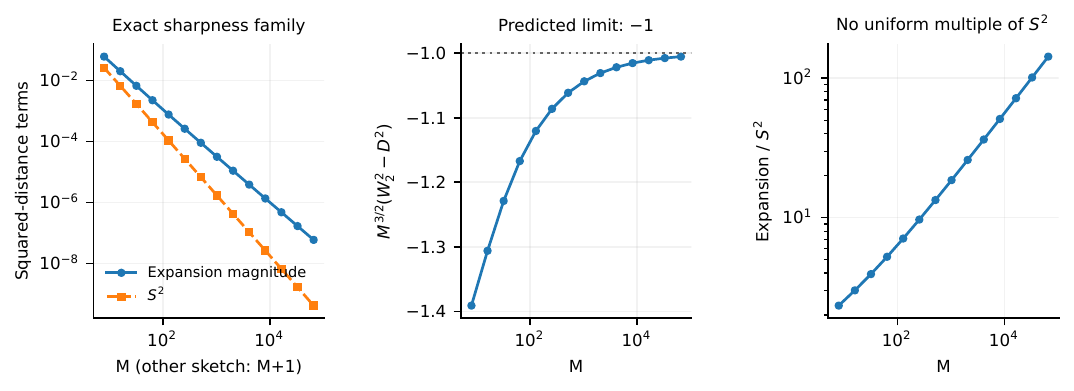}
\caption{Exact integration of the sharpness family. The expansion decays more
slowly than the residual-only budget, its $M^{3/2}$ rescaling approaches $-1$,
and its ratio to $S^2$ grows. The target $\nu_M$ varies with $M$.}
\label{fig:theory-sharpness}
\end{figure}

\subsection{Moment fidelity, correction and recovery}
\paragraph{Finite scalar and tensor bounds (\cref{prop:moments}).}
We use three weighted clouds of $181$ bounded points in $\R^3$, $30$ generic
directions, $M=2,4,8,16$, and optimal, equal-mass and nonuniform prescribed
cells. Exact atom/cell intersections give per-cell energies and the actual
largest residual $h_\theta$. The \TVScalarChecks{} scalar checks cover means,
second moments, raw and corrected orders $2$--$6$, even-moment signs, and the
$C^2$ bound for $x^2$, $\sin(1.7x)$, $\cos(0.7x)$ and $\sqrt{1+x^2}$.
The largest centroid-identity discrepancy is \TVMomentIdentityError; all
inequalities pass, and the largest corrected-error/bound ratio is
\TVCorrectedMomentRatio. For the \TVTensorChecks{} tensor configurations of
orders $2$--$4$, the largest uncorrected and corrected ratios are
\TVTensorRawRatio{} and \TVTensorCorrectedRatio. Tensor coordinates use the
Frobenius-orthonormal symmetric basis with factors $\sqrt{r!/\alpha!}$;
the norm is checked independently against explicitly accumulated full tensors.

\paragraph{Moment rates.}
Let $U$ be uniform on $[0,1]$ and
$X(U)=(U+0.3U^2,\,0.2+0.5U+0.4U^2)$. Twelve directions with angles equally
spaced between $0.06$ and $\pi/2-0.06$ give increasing polynomial quantiles.
We use equal-mass cells for $M=2,4,\ldots,128$, and degree-exact Gaussian
quadrature for all moments. The code directly checks
$\varepsilon_\theta^2=O(M^{-2})$ and $h_\theta=O(M^{-1})$ on this range.
For the fourth directional moment, fitted error slopes over $M=16,32,64,128$
are \TVScalarRawSlope{} without correction and \TVScalarCorrectedSlope{}
with correction; the tensor slopes are \TVTensorRawSlope{} and
\TVTensorCorrectedSlope{}. The plotted theoretical budgets scale as
$M^{-2}$ and $M^{-3}$, respectively. The observed $M^{-4}$ corrected error
is faster than the asserted upper bound, as can occur for smooth cells.

\paragraph{Exact recovery, rank and conditioning.}
For an unbounded three-dimensional Gaussian with prescribed mean and positive
definite covariance, analytic truncated-normal cell integrals include the two
infinite tail cells. With $L=6,12,18$ directions and $M=2,4,8,16,32$, adding
one residual energy per slice recovers the covariance with maximum Frobenius
error \TVCovarianceError{} against the independently specified matrix.
Three spanning directions recover the mean with error \TVMeanError.
The \TVRankChecks{} random-direction checks for tensor orders $1$--$4$
give rank $\min\{L,\binom{d+r-1}{r}\}$. A saved counterexample below the
covariance threshold consists of two positive definite Gaussian covariances
at Frobenius distance $0.5$ with identical variances on five directions up
to rounding. Finally, nearly aligned directions remain full rank but amplify
perturbations; perturbing in the least singular direction reproduces
$\norm{A_2^\dagger}_{2\to F}$. The witness matrices, directions and conditioning
data are saved.

\begin{figure}[!htbp]
\centering
\includegraphics[width=\textwidth]{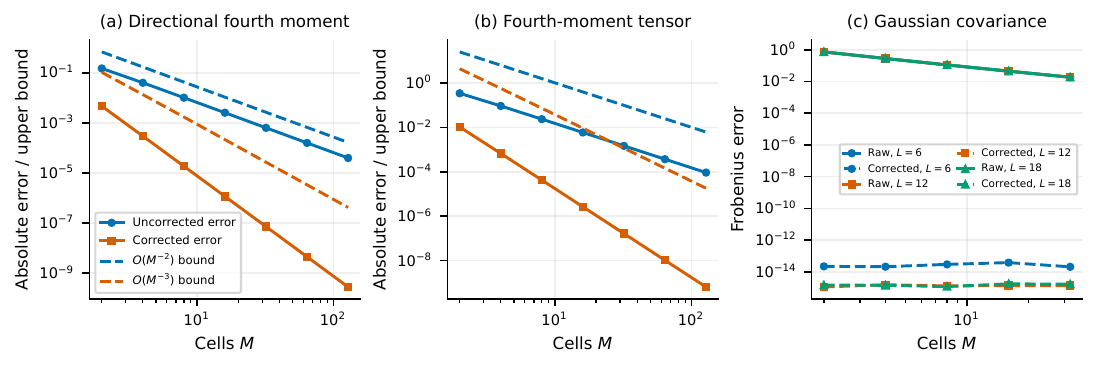}
\caption{Moment checks. Left and middle: fourth-order scalar and tensor errors
with the bounds of \cref{prop:moments}; corrected errors decrease faster than
their $M^{-3}$ bounds in this smooth example. Right: covariance errors before
and after adding residual energies for an unbounded Gaussian.}
\label{fig:theory-moments}
\end{figure}

%% file: results/theory_validation/numbers.tex
\newcommand{\TVIdentityChecks}{\ensuremath{480}}
\newcommand{\TVIdentityError}{\ensuremath{1.29\times10^{-15}}}
\newcommand{\TVMaxViolation}{\ensuremath{4.89\times10^{-16}}}
\newcommand{\TVExpansions}{\ensuremath{142}}
\newcommand{\TVContractions}{\ensuremath{230}}
\newcommand{\TVSlicedChecks}{\ensuremath{48}}
\newcommand{\TVSamplingChecks}{\ensuremath{64}}
\newcommand{\TVSamplingRatio}{\ensuremath{0.364}}
\newcommand{\TVDirectionReplicates}{\ensuremath{8,192}}
\newcommand{\TVDirectionSquaredSlope}{\ensuremath{-0.504}}
\newcommand{\TVDirectionSlope}{\ensuremath{-0.506}}
\newcommand{\TVScalarChecks}{\ensuremath{17,280}}
\newcommand{\TVTensorChecks}{\ensuremath{108}}
\newcommand{\TVMomentIdentityError}{\ensuremath{9.44\times10^{-16}}}
\newcommand{\TVCorrectedMomentRatio}{\ensuremath{0.192}}
\newcommand{\TVTensorRawRatio}{\ensuremath{0.558}}
\newcommand{\TVTensorCorrectedRatio}{\ensuremath{0.0363}}
\newcommand{\TVCovarianceError}{\ensuremath{3.78\times10^{-14}}}
\newcommand{\TVMeanError}{\ensuremath{2.29\times10^{-15}}}
\newcommand{\TVRankChecks}{\ensuremath{144}}
\newcommand{\TVScalarRawSlope}{\ensuremath{-1.999}}
\newcommand{\TVScalarCorrectedSlope}{\ensuremath{-4.000}}
\newcommand{\TVTensorRawSlope}{\ensuremath{-1.999}}
\newcommand{\TVTensorCorrectedSlope}{\ensuremath{-4.000}}
\newcommand{\TVRootResidual}{\ensuremath{1.00\times10^{-13}}}
\newcommand{\TVQuadratureChange}{\ensuremath{4.34\times10^{-12}}}
\newcommand{\TVSolverChange}{\ensuremath{6.56\times10^{-15}}}
\newcommand{\TVSharpM}{\ensuremath{65,536}}
\newcommand{\TVSharpLimit}{\ensuremath{-1.005526}}
\newcommand{\TVSharpRatio}{\ensuremath{143.12}}
\newcommand{\TVDensityConstant}{\ensuremath{1247.6}}
\newcommand{\TVDensityScaledGap}{\ensuremath{0.17901}}
\newcommand{\TVWidthScaled}{\ensuremath{1.1864}}
\newcommand{\TVMassScaled}{\ensuremath{1.3446}}
\newcommand{\TVEqualBallTwoError}{\ensuremath{0.40858}}

%% file: results/theory_validation/rates_table.tex
\begin{tabular}{@{}lccc@{}}
\toprule
Law & $M^2\varepsilon_*^2/C_{\rm Zador}$ & Equal-mass normalization & Rescaled equal-mass error\\
\midrule
Gaussian & 0.99597 & $M\log M$ & 1.07909\\
Ball d=1 & 1.00000 & $M^2$ & 0.33333\\
Ball d=2 & 0.99939 & $M^2$ & 0.40858\\
Ball d=3 & 0.99892 & $M^2/\log M$ & 0.07224\\
Ball d=5 & 0.99827 & $M^{10/6}$ & 0.11563\\
Ball d=9 & 0.99755 & $M^{14/10}$ & 0.03175\\
\bottomrule
\end{tabular}